\documentclass[12pt]{article}
\usepackage{amsmath}
\usepackage{times}
\usepackage{graphicx}
\usepackage{color}
\usepackage{multirow}
\usepackage[authoryear]{natbib}

\usepackage{rotating}
\usepackage{bbm}
\usepackage{latexsym}
\usepackage{microtype}
\usepackage{subfigure}
\usepackage{booktabs}
\usepackage[T1]{fontenc}
\usepackage{amsmath,amsthm,amssymb,amsbsy}
\usepackage{textcomp}
\usepackage{comment}
\usepackage{pdflscape}
\usepackage{array}
\usepackage{algorithm}
\usepackage{algorithmic}

\usepackage[symbol]{footmisc}

\newtheorem{theorem}{Theorem}
\newtheorem{lemma}[theorem]{Lemma}

\usepackage{url}
\usepackage{xcolor}
\definecolor{mydarkred}{rgb}{0.6,0,0}
\definecolor{mydarkgreen}{rgb}{0,0.6,0}
\usepackage[colorlinks,
 	linkcolor=mydarkred,
 	citecolor=mydarkgreen]{hyperref}

\DeclareMathOperator*{\maximize}{maximize}

\DeclareMathOperator*{\argmin}{argmin}
\DeclareMathOperator*{\argmax}{argmax}

\newcommand{\rP}{\mathrm{P}}
\newcommand{\rN}{\mathrm{N}}

\newcommand{\rU}{\mathrm{U}}

\newcommand{\thetap}{\theta_\mathrm{P}^{}}
\newcommand{\thetan}{\theta_\mathrm{N}^{}}

\newcommand{\bxp}{\boldsymbol{x}^\mathrm{P}}

\newcommand{\bxu}{\boldsymbol{x}^\mathrm{U}}
\newcommand{\np}{{n_\mathrm{P}^{}}}

\newcommand{\nun}{{n_\mathrm{U}^{}}}

\newcommand{\bb}{{\boldsymbol{b}}}

\newcommand{\bx}{{\boldsymbol{x}}}

\newcommand{\bh}{{\boldsymbol{h}}}
\newcommand{\bH}{{\boldsymbol{H}}}
\newcommand{\bI}{{\boldsymbol{I}}}

\newcommand{\bu}{{\boldsymbol{u}}}
\newcommand{\bU}{{\boldsymbol{U}}}
\newcommand{\bv}{{\boldsymbol{v}}}

\newcommand{\bbeta}{{\boldsymbol{\beta}}}

\newcommand{\bmu}{{\boldsymbol{\mu}}}
\newcommand{\bSigma}{{\boldsymbol{\Sigma}}}
\newcommand{\bphi}{{\boldsymbol{\phi}}}

\newcommand{\bzero}{{\boldsymbol{0}}}

\newcommand{\bbetah}{{\boldsymbol{\widehat{\beta}}}}

\newcommand{\bhhp}{\boldsymbol{\widehat{h}}{}^\rP}
\newcommand{\bHhu}{\boldsymbol{\widehat{H}}{}^\rU}

\newcommand{\calB}{\mathcal{B}}

\newcommand{\calF}{\mathcal{F}}

\newcommand{\calO}{\mathcal{O}}

\newcommand{\calU}{\mathcal{U}}
\newcommand{\calW}{\mathcal{W}}

\newcommand{\dbx}{\mathrm{d}\boldsymbol{x}}

\newcommand{\bbR}{\mathbb{R}}
\newcommand{\bbS}{\mathbb{S}}

\newcommand{\ph}{\widehat{p}}
\newcommand{\gh}{\widehat{g}}

\newcommand{\wh}{\widehat{w}}
\newcommand{\Jh}{\widehat{J}}

\newcommand{\PU}{\mathrm{PU}}

\newcommand{\SMI}{\mathrm{SMI}}

\newcommand{\PUSMI}{\mathrm{PU}\textrm{-}\mathrm{SMI}}
\newcommand{\PUSMIh}{\widehat{\PUSMI}}
\newcommand{\PNSMI}{\mathrm{PN}\textrm{-}\mathrm{SMI}}
\newcommand{\PNSMIh}{\widehat{\PNSMI}}

\renewcommand{\subfigtopskip}{0mm}
\renewcommand{\subfigbottomskip}{0mm}
\renewcommand{\subfigcapskip}{0mm}
\renewcommand{\subcapsize}{\small}

\usepackage{authblk}
\title{
	Information-Theoretic Representation Learning 
	for Positive-Unlabeled Classification
	}
		
\author[1,2]{Tomoya Sakai\footnote{The affiliation is as of March 2018.}}
\newcommand\NoteMark{\footnotemark[\arabic{footnote}]} % get the current value
\author[1,2]{Gang Niu\protect\NoteMark}
\author[2,1]{Masashi Sugiyama\protect\NoteMark}

\affil[1]{Graduate School of Frontier Sciences, \protect\\ 
	The University of Tokyo, Japan}
\affil[2]{Center for Advanced Intelligence Project, \protect\\ 
	RIKEN, Japan}

\date{}

\begin{document}
\sloppy
\maketitle

\begin{abstract}
Recent advances in weakly supervised classification
allow us to train a classifier only
from \emph{positive and unlabeled} (PU) data.
However, existing PU classification methods typically require
an accurate estimate of the class-prior probability,
which is a critical bottleneck
particularly for high-dimensional data.
This problem has been commonly addressed
by applying principal component analysis in advance,
but such unsupervised dimension reduction can collapse
underlying class structure.
In this paper, we propose a novel representation learning method from PU data
based on the \emph{information-maximization principle}.
Our method does not require class-prior estimation
and thus can be used as a preprocessing method for PU classification.
Through experiments, we demonstrate that our method
combined with deep neural networks highly improves the accuracy of PU class-prior
estimation, leading to state-of-the-art PU classification performance.
\end{abstract}

\section{Introduction}
\label{sec:intro}

In real-world applications, 
it is conceivable that only \emph{positive and unlabeled} (PU) data
are available for training a classifier.
For instance, in land-cover image classification, 
images of urban regions can be easily labeled,
while images of non-urban regions are difficult to 
annotate due to high diversity of non-urban regions
containing, e.g., forest, seas, grasses, and soil
\citep{TGRS:Li+etal:2011}. 
To cope with such situations, PU classification has been actively studied
\citep{ALT:Letouzey2000,KDD:Elkan+Noto:2008,ICML:duPlessis+etal:2015},
and the state-of-the-art method
allows us to systematically train deep neural networks only from PU data
\citep{NIPS:Kiryo+etal:2017}.

However, existing PU classification methods
typically require an estimate of the class-prior probability,
and their performance is sensitive to the quality of 
class-prior estimation \citep{NIPS:Kiryo+etal:2017}.
Although various class-prior estimation methods from PU data have been proposed
so far \citep{IEICE:duPlessis+Sugiyama:2014,ICML:Ramaswamy+etal:2016,NIPS:Jain+etal:2016,ML:duPlessis+etal:2017,UAI:Norhcutt+etal:2017},
accurate estimation of the class-prior
is still highly challenging particularly for high-dimensional data.

In practice, principal component analysis is commonly used 
to reduce the data dimensionality in advance
\citep{ICML:Ramaswamy+etal:2016,ML:duPlessis+etal:2017}.
However, such unsupervised dimension reduction completely abandons label information
and thus the underlying class structure may be smashed.
As a result,  class-prior estimation often becomes even more difficult
after dimension reduction.

The goal of this paper is to cope with this problem by proposing
a representation learning method that can be executed
only from PU data.
Our method is developed within the framework of \emph{information maximization}
\citep{Comp:Linsker:1988}.

\emph{Mutual information} (MI) \citep{book:Cover+Thomas:2006}
is a statistical dependency measure between random variables
that is popularly used in information-theoretic machine learning
\citep{JMLR:Torkkola:2003,NIPS:Krause+etal:2010}.
However, empirically approximating MI from continuous-valued training data is
not straightforward \citep{PhysRevE:Moon:1995,PhysRevE:Kraskov+etal:2004,PRE:Khan+etal:2007,NC:Hulle:2005,FSDM:Suzuki+etal:2008}
and is often sensitive to outliers \citep{Biometrika:Basu+etal:1998,AISM:Sugiyama+etal:2012}.
For this reason, we employ a squared-loss variant 
of mutual information (SMI) \citep{BMCBio:Suzuki+etal:2009a,Entropy:Sugiyama:2013},
whose empirical estimator is known to be robust to outliers
and possess superior numerical properties \citep{MLJ:Kanamori+etal:2012}.

Our contributions are summarized as follows:
\begin{itemize}
\item We first develop a novel estimator of SMI that can be computed only from PU data,
  and prove its convergence to the optimal estimate of SMI in the optimal
  parametric rate when the linear-in-parameter model is used (Section~\ref{sec:PU-SMI}).
\item Based on this PU-SMI estimator, we then propose 
  a representation learning method that can be executed \emph{without}
  estimating the class-prior probabilities of unlabeled data
  (Section~\ref{sec:pu-rep-learn}).
\item Finally, we experimentally demonstrate that
  our PU representation learning method combined with deep neural networks
  highly improves the accuracy of PU class-prior estimation,
  and consequently the accuracy of PU classification can also be boosted
  significantly (Section~\ref{sec:experiments}).
\end{itemize}

\section{SMI}
In this section, we review the definition of ordinary MI and its variant, SMI.

Let $\bx\in\mathbb{R}^d$ be an input pattern, 
$y\in\{\pm 1\}$ be a corresponding class label,
and $p(\bx,y)$ be the underlying joint density,
where $d$ is a positive integer. 

\emph{Mutual information} (MI) \citep{book:Cover+Thomas:2006}
is a statistical dependency measure defined as
\begin{align*}
\mathrm{MI}&:=\sum_{y=\pm1}
 	\int p(\bx,y)\log\left(\frac{p(\bx,y)}{p(\bx)p(y)}\right)
             \dbx,
\end{align*}
where $p(\bx)$ is the marginal density of $\bx$ 
and $p(y)$ is the probability mass of $y$.  
MI can be regarded as the \emph{Kullback-Leibler divergence}
from $p(\bx,y)$ to $p(\bx)p(y)$,
and therefore MI is non-negative and takes zero if and only if $p(\bx,y)=p(\bx)p(y)$,
i.e., $\bx$ and $y$ are statistically independent.
This property allows us to evaluate the dependency between $\bx$ and $y$.
However, empirically approximating MI from continuous data is
not straightforward \citep{PhysRevE:Moon:1995,PhysRevE:Kraskov+etal:2004,PRE:Khan+etal:2007,NC:Hulle:2005,FSDM:Suzuki+etal:2008}
and is often sensitive to outliers \citep{Biometrika:Basu+etal:1998,AISM:Sugiyama+etal:2012}.

To cope with this problem, \emph{squared-loss MI} (SMI) has been proposed
\citep{BMCBio:Suzuki+etal:2009a},
which is a squared-loss variant of MI defined as
\begin{align}
\SMI&:=\sum_{y=\pm1}\frac{p(y)}{2}
	\int\Big(\frac{p(\bx,y)}{p(\bx)p(y)}-1
	\Big)^2p(\bx)\dbx. 
	\label{eq:smi-def}
\end{align}
SMI can be regarded as the \emph{Pearson divergence}
\citep{PhMag:Pearson:1900} from $p(\bx,y)$ to $p(\bx)p(y)$.
SMI is also non-negative and takes zero if and only if $\bx$ and $y$
are independent.

So far, methods for estimating SMI from positive and negative samples
and SMI-based machine learning algorithms have been explored extensively,
and their effectiveness has been demonstrated \citep{Entropy:Sugiyama:2013}.

\section{SMI Estimation from PU Data}
\label{sec:PU-SMI}
The goal of this paper is to develop a representation learning method from PU data.
To this end, we propose an estimator of SMI that can be computed only from PU
data in this section.

\subsection{SMI with PU Data}
Suppose that we are given PU data \citep{Biometrics:Ward+etal:2009}:
\begin{align*}
\{\bxp_i\}^\np_{i=1} &\stackrel{\mathrm{i.i.d.}}{\sim}
p(\bx\mid y=+1) , \\
\{\bxu_k\}^\nun_{k=1} &\stackrel{\mathrm{i.i.d.}}{\sim}
p(\bx)=\thetap p(\bx\!\mid\! y=+1) 
	 + \thetan p(\bx\!\mid\! y=-1) , 
\end{align*}
where $\thetap:=p(y=+1)$ and $\thetan:=p(y=-1)$ 
are the class-prior probabilities.

First, we express SMI in Eq.~\eqref{eq:smi-def}
in terms of only the densities of PU data, without negative data 
(see Appendix~\ref{app:theorem:PUSMI=SMI} for its proof):
\begin{theorem}\label{theorem:PUSMI=SMI}
Let
\begin{align}
\PUSMI:=\frac{\thetap}{2\thetan}
	\int\Big(\frac{p(\bx\mid y=+1)}{p(\bx)}-1\Big)^2  p(\bx)\dbx. 
\label{PU-SMI}
\end{align}
Then we have $\PUSMI=\SMI$.
\end{theorem}

If PU densities $p(\bx\mid y=+1)$ and $p(\bx)$ are estimated from PU data,
the above $\PUSMI$ allows us to approximate SMI only from PU data.
However, such a naive approach works poorly due to 
hardness of density estimation and computing the ratio of estimated densities
further magnifies the estimation error
\citep{book:Sugiyama+etal:2012}.

\subsection{PU-SMI Estimation}
\label{sec:pusmi-est}
Here, we propose a more sophisticated approach to estimating $\PUSMI$
from PU data.

First, we give the following theorem,
which gives a lower-bound of $\PUSMI$
(see Appendix~\ref{app:theorem:PUSMI-lowerbound} for its proof):
\begin{theorem}\label{theorem:PUSMI-lowerbound}
For any function $f(\bx)$,
\begin{align}
\PUSMI&\ge\frac{\thetap}{\thetan}
\Big(-J_\PU(f)-\frac{1}{2}\Big),
\label{PUSMI-lowerbound}
\end{align} 
where
\begin{align*}
  J_\PU(f):=\frac{1}{2}\int f^2(\bx)p(\bx)\dbx - \int f(\bx)p(\bx\mid y=+1)\dbx,
\end{align*}
and the equality holds if and only if
\begin{align*}
f(\bx)=\frac{p(\bx\mid y=+1)}{p(\bx)} .
\end{align*}
\end{theorem}

While $\PUSMI$ itself contains $p(\bx\mid y=+1)$ and $p(\bx)$
in a complicated way, the lower bound 
consists only of the expectations over $p(\bx\mid y=+1)$ and $p(\bx)$.
Thus, the lower bound can be immediately approximated empirically.

Based on this theorem, we maximize
an empirical approximation to the lower bound \eqref{PUSMI-lowerbound},
which is expressed as
\begin{align*}
\wh:=\argmin_{w\in\mathcal{W}}\; \Jh_\PU(w),
\end{align*}
where
\begin{align}
\Jh_\PU(w):=\frac{1}{2\nun}\sum^\nun_{k=1}w^2(\bxu_k)
	-\frac{1}{\np}\sum^\np_{i=1}w(\bxp_i) 
	\label{eq:pusmi-loss} 	
\end{align}
and $\mathcal{W}$ is a (user-defined) function class such as linear-in-parameter models,
kernel models, and neural networks.
In this optimization, we can drop
the unknown class-prior ratio ${\thetap}/{\thetan}$,
which is difficult to estimate accurately
 \citep{IEICE:duPlessis+Sugiyama:2014,ICML:Ramaswamy+etal:2016,NIPS:Jain+etal:2016,ML:duPlessis+etal:2017}.

Finally, our PU-SMI estimator is given as
\begin{align}
\PUSMIh=\frac{\thetap}{\thetan}\Big(-\Jh_\PU(\wh)- \frac{1}{2}\Big) .
\label{eq:pu-smi-loss-form}
\end{align}
$\PUSMIh$ includes the class-prior ratio $\thetap/\thetan$
only as a proportional constant.
Therefore, class-prior estimation is not needed
when we just want to maximize or minimize PU-SMI. 
We will utilize this excellent property in Section~\ref{sec:pu-rep-learn}
when we develop a representation learning method.

Note that if $W$ contains the true density-ratio function $p(\bx\mid y=+1)/p(\bx)$,
$\PUSMIh\to\PUSMI\; (\np,\nun\to\infty)$ with some regularity condition.  
On the other hand, if the function class does not contain the true density-ratio function,
there is a gap between $\PUSMI$ and $\PUSMIh$ even if $\np,\nun\to\infty$. 
Such a gap often arises in real-world applications because a function class
does not always include the true density-ratio function.
However, the gap may not be a critical issue in practice
as long as a reasonably flexible function class is chosen,
as demonstrated by the experiments in Section~\ref{sec:experiments}. 
Even though the gap may exist in practical implementation,   
we show that the classification performance can be improved by 
our proposed representation learning method.

\subsection{Analytic Solution for Linear-in-Parameter Models}
Our SMI estimator is applicable to any density-ratio model $w$.

If a neural network is used as $w$, the solution may be obtained by 
a stochastic gradient method
\citep{BOOK:Goodfellow+etal:2016,Tensorflow:2015,Caffe:Jia+etal:2014}.

Another candidate of the density-ratio model is a linear-in-parameter model:
\begin{align}
w(\bx)=\sum^b_{\ell=1}\beta_\ell\phi_\ell(\bx)
	=\bbeta^\top\bphi(\bx) ,
	\label{eq:lip-model}
\end{align}
where $\bbeta:=(\beta_1,\ldots,\beta_b)^\top\in\mathbb{R}^b$
is a vector of parameters, $^\top$ denotes the transpose,
$b$ is the number of parameters,
and
$\bphi(\bx):=(\phi_1(\bx),\ldots,\phi_b(\bx))^\top\in\mathbb{R}^b$
is a vector of basis functions.
This model allows us to obtain an analytic-form PU-SMI estimator.
Furthermore, the optimal convergence is theoretically guaranteed
as shown in Section~\ref{sec:theory}.

When the $\ell_2$-regularizer is included,
the optimization problem yields
\begin{align*}
\bbetah:=\argmin_{\bbeta}\; \frac{1}{2}\bbeta^\top\bHhu\bbeta
	-\bbeta^\top\bhhp + \frac{\lambda_\PU}{2}\|\bbeta\|_2^2,
\end{align*}
where $\lambda_\PU\ge0$ is the regularization parameter,
$\|\cdot\|_2$ denotes the $\ell_2$-norm, and 
\begin{align*}
\widehat{H}_{\ell,\ell'}^\rU&:=\frac{1}{\nun}\sum^\nun_{k=1} 
	\phi_\ell(\bxu_k)\phi_{\ell'}(\bxu_k), \\ 
\widehat{h}_\ell^\rP&:=\frac{1}{\np}\sum^\np_{i=1}\phi_\ell(\bxp_i) .
\end{align*}
Note that $\widehat{H}_{\ell,\ell'}^\rU$ is the $(\ell, \ell')$-th element of 
$\bHhu$ and $\widehat{h}_{\ell}^\rP$ is the $\ell$-th element of $\bhhp$. 
The solution can be obtained analytically 
by differentiating the objective function with respect to $\bbeta$
and set it to zero: $\bbetah=(\bHhu+\lambda_\PU\bI_b)^{-1}\bhhp$.
Finally, with the obtained estimator, 
we can compute an SMI approximator only from positive 
and unlabeled data:
\begin{align*}
\PUSMIh=\frac{\thetap}{\thetan}\Big(\bbetah^\top\bhhp
	-\frac{1}{2}\bbetah^\top\bHhu\bbetah
	-\frac{1}{2}\Big) .		
\end{align*}
  
Note that all hyper-parameters such as the regularization parameter
can be tuned by the value of $J_\PU$ approximated by  (cross-)validation samples.

\subsection{Convergence Analysis}
\label{sec:theory}
Here we analyze the convergence rate of 
learned parameters of the density-ratio model and the PU-SMI approximator
based on the \emph{perturbation analysis of optimization problems} 
\citep{SICON:Bonnans+Cominetti:1996,SIAM:Bonnans+Shapiro:1998}.

In our theoretical analysis, we focus on the linear-in-parameter model 
in Eq.~\eqref{eq:lip-model}.
We first define $\bbeta^{\ast\top}_{}\bphi(\bx)$ as the minimizer 
of the expected error, i.e., 
$\bbeta^\ast_{}:=\argmin_{\bbeta\in\bbR^b}\; J_\PU(\bbeta)$
and denote its estimator by 
$\bbetah=\argmin_{\bbeta\in\bbR^b}\; \Jh_\PU(\bbeta)$
in this analysis.
Note that the linear-in-parameter model is assumed as a simple baseline for
theoretical analysis.

For the linear-in-parameter model, we assume that the basis functions
satisfy $0\leq\phi_\ell(\bx)\leq 1$ for all $\ell=1,\ldots,b$,
and $\bHhu$ and $\bH^\rU$ are positive definite matrices.

Let 
\begin{align*}
\PUSMI^\ast&:=\frac{\thetap}{\thetan}
	\Big(-J_\PU(\bbeta^\ast)-\frac{1}{2}\Big)
\end{align*}
be the PU-SMI with $\bbeta^\ast$.
Similarly, $\PUSMIh$ is the estimate of the PU-SMI with $\bbetah$.
Let $\calO_p$ denote the order in probability.
Then we have the following convergence results
(its proof is given in Appendix~\ref{app:proof-rate}):
\begin{theorem}
\label{thm:rate}
As $\np,\nun\to\infty$, we have
\begin{align*}
\|\bbetah - \bbeta^\ast\|_2&=\calO_p(1/\sqrt{\np}+1/\sqrt{\nun}) , \\ 
|\PUSMI^\ast - \widehat{\PUSMI}|&=\calO_p(1/\sqrt{\np}+1/\sqrt{\nun}) .
\end{align*}
\end{theorem}

Theorem~\ref{thm:rate} guarantees that 
the convergence of the density-ratio estimator
and the PU-SMI approximator. 
In our setting, since $\np$ and $\nun$ can increase independently,
this is the optimal convergence rate without any additional 
assumption \citep{JMLR:Kanamori+etal:2009,MLJ:Kanamori+etal:2012}.

Theorem~\ref{thm:rate} shows
that both positive and unlabeled data contribute to
convergence. This implies that unlabeled data is directly used 
in the estimation rather than extracting the information of a data structure,
such as the cluster structure 
frequently assumed in semi-supervised learning \citep{book:Chapelle+etal:2006}.
The theorem also shows that the convergence rate of our method is dominated by 
the smaller size of positive or unlabeled data.

Note that since this analysis focuses on the linear-in-parameter model, 
there might be a gap between $\PUSMI$ and $\PUSMI^\ast$,
implying that $\PUSMI<\PUSMI^\ast$. 
The convergence analysis guarantees that $\PUSMIh$ with 
the linear-in-parameter model converges to $\PUSMI^\ast$,
but there might be an approximation error \citep{book:Mohri+etal:2012},
as discussed in Section~\ref{sec:pusmi-est}.  
   
\section{PU Representation Learning}
\label{sec:pu-rep-learn}
In this section, we propose a representation learning 
method based on PU-SMI maximization.
We extend the existing SMI-based dimension reduction \citep{NeCo:Suzuki:2013},
called \emph{least-squares dimension reduction} (LSDR),
to PU representation learning.
While LSDR only considers linear dimension reduction,
we extend it to non-linear dimension reduction by neural networks.
  
Let $\bv\colon\bbR^d\to\bbR^m$, where $m<d$, 
be a mapping from an input vector to 
its low-dimensional representation.
If the mapping function satisfies
\begin{align}
p(y\mid\bx)=p(y\mid \bv(\bx)) ,
\label{eq:sufficient-dim-red}
\end{align}
the obtained low-dimensional representation can be used
as the new input instead of the original input vector.
Finding the mapping function satisfying 
the condition~\eqref{eq:sufficient-dim-red} 
is known as \emph{sufficient dimension reduction} \citep{JASA:Li:1991}.
Let $\widetilde{\SMI}$ be SMI between $\bv(\bx)$ and $y$.
\citet{NeCo:Suzuki:2013} proved $\SMI\geq\widetilde{\SMI}$
and equality holds when the condition~\eqref{eq:sufficient-dim-red}
is satisfied. 
That is, maximizing SMI is finding \emph{sufficient} representation
for the output $y$.

Following the information-maximization principle \citep{Comp:Linsker:1988},
we maximize PU-SMI with respect to the mapping
to find low-dimensional representation
that maximally preserves dependency between input and output. 
  
\begin{algorithm}[t]
\caption{PU Representation Learning}
\label{alg:pudr}
	\begin{algorithmic}[1]
 	\REQUIRE $\{\bxp_i\}^\np_{i=1}$, and $\{\bxu_k\}^\nun_{k=1}$.
 	\STATE Initialize $\widehat{g}$ and $\widehat{\bv}$ 
	\REPEAT
		\STATE Update $\widehat{g}\leftarrow \widehat{g} 
			- \varepsilon_g \nabla_{g}\Jh_\PU(g\circ \widehat{\bv})\big|_{g=\widehat{g}}$
		\STATE Update $\widehat{\bv}\leftarrow \widehat{\bv} 			
			- \varepsilon_{\bv}\nabla_{\bv}\Jh_\PU(\widehat{g}\circ \bv)\big|_{\bv=\widehat{\bv}}$
	\UNTIL{stopping conditions meet}
	\RETURN $\widehat{g}$ and $\widehat{\bv}$. 
	\end{algorithmic}
\end{algorithm}

More specifically, since 
$\PUSMIh=-\thetap/\thetan \cdot \min_{w\in\calW}\;\Jh_\PU(w)-\thetap/(2\thetan)$,
we minimize $\Jh_\PU(w)$ with respect to $w$.
Furthermore, inspired by the alternative optimization algorithm
for the SMI-based dimension reduction method \citep{NeCo:Suzuki:2013},
we decompose $w$ into $g$ and $\bv$ such that $w=g\circ \bv$
and minimize $\Jh_\PU(g\circ \bv)$ by optimizing $g$ and $\bv$ alternatively,
where $g\colon\bbR^m\to\bbR$ and
``$\circ$'' denotes the function composition, i.e.,
$(g \circ \bv)(\bx)=g(\bv(\bx))$.
In this decomposition, $\bv$ can be regarded as a mapping function extracting
features from input pattern and $g$ a density ratio function $p(\bx\mid y=+1)/p(\bx)$.   
First, we approximate SMI by minimizing
Eq.~\eqref{eq:pusmi-loss} with respect to density ratio $g$
with current mapping $\widehat{\bv}$ fixed:
\begin{align*}
\gh=\argmin_{g}\; \Jh_\PU(g\circ\widehat{\bv}) .
\end{align*}
Then, we update mapping $\widehat{\bv}$ to increase the estimated PU-SMI
with current density ratio $\gh$ fixed:
\begin{align*}
\widehat{\bv}\leftarrow \widehat{\bv} 
  - \varepsilon\nabla_{\bv}\Jh_\PU(\gh\circ \bv),
\end{align*}
where $\varepsilon$ is the step size.
This process is repeated until convergence.
In practice, we may alternately optimize $g$ and $\bv$ as described in
Algorithm~\ref{alg:pudr} to simplify the implementation.\footnote{
We also tried to optimize $w$ and $\bv$ simultaneously.
That is, $J_\PU$ is minimized with respect to $g$ without decomposing $g$
into $w$ and $\bv$,
but it did not work well in our preliminary experiments.
}
We refer to our representation learning method for PU data as
\emph{positive-unlabeled representation learning (PURL)}.

Note again that, in the above optimization process,
unknown class-prior ratio ${\thetap}/{\thetan}$ does not need to be estimated in advance,
which is a significant advantage of the proposed method.

\section{Experiments}
\label{sec:experiments}
In this section, we experimentally investigate the behavior of the proposed PU-SMI 
estimator and evaluate 
the performance of the proposed representation learning method
on various benchmark datasets.

\subsection{Accuracy of PU-SMI Estimation}
\label{sec:exp-est}
First, we investigate the estimation accuracy of the proposed 
PU-SMI estimator on datasets obtained from
the \emph{LIBSVM} webpage \citep{LibSVM:Chang+etal:2011}.

As the model $w$,
we use the linear-in-parameter model with the Gaussian basis functions 
$\phi_\ell(\bx):=\exp(-\|\bx-\bx_\ell\|^2/(2\sigma^2))$ for $\ell=1,\ldots,b$,
where $\sigma>0$ is the bandwidth and $\{\bx_\ell\}^b_{\ell=1}$
are the centers of the Gaussian functions 
randomly sampled from $\{\bxu_k\}^\nun_{k=1}$.
The Gaussian bandwidth and the $\ell_2$-regularization parameter are
determined by five-fold cross-validation.
We vary the number of positive/unlabeled samples
from $10$ to $200$, with the number of unlabeled/positive samples fixed.
The class-prior was assumed to be known in this illustrative experiment
and set at $\thetap=0.5$.

\begin{figure*}[t]
	\centering	
	\subfigure[$\nun=400$]{%
		\includegraphics[clip, width=.48\columnwidth]{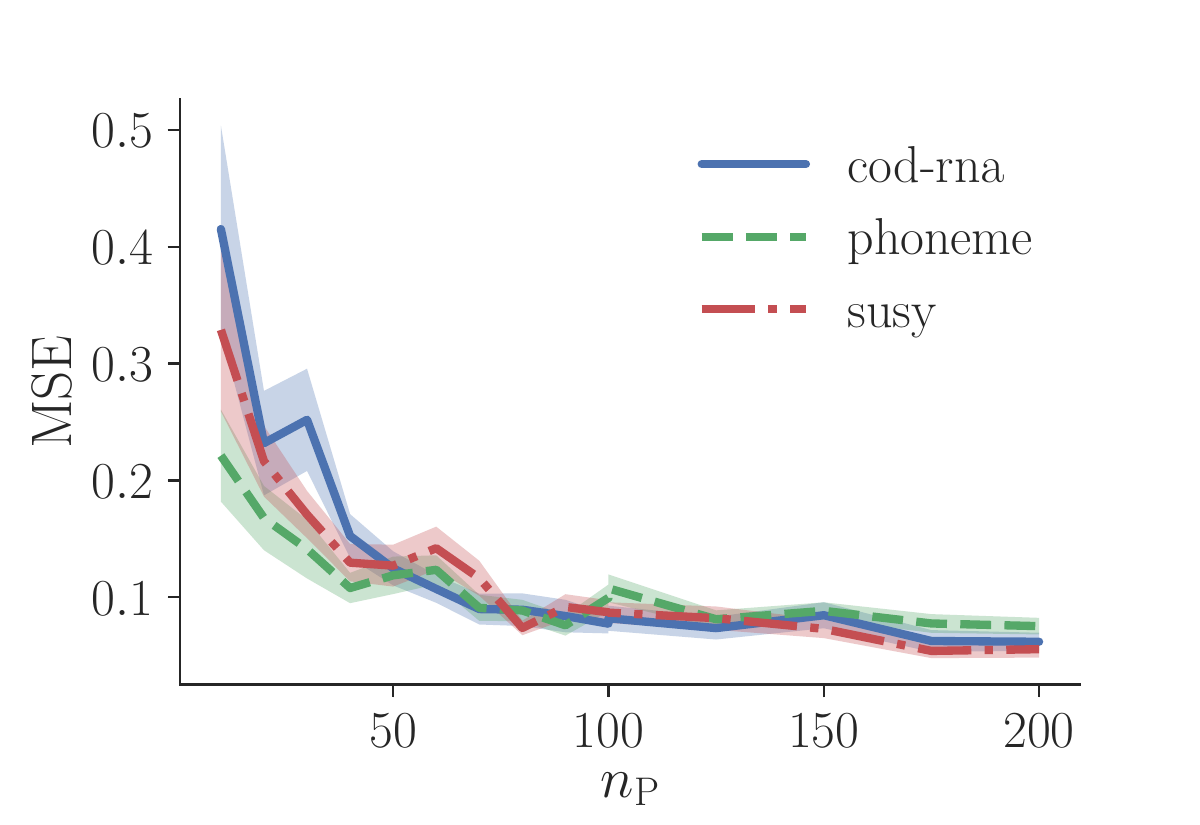}
	}%
    \subfigure[$\np=200$]{%    
    	\includegraphics[clip, width=.48\columnwidth]{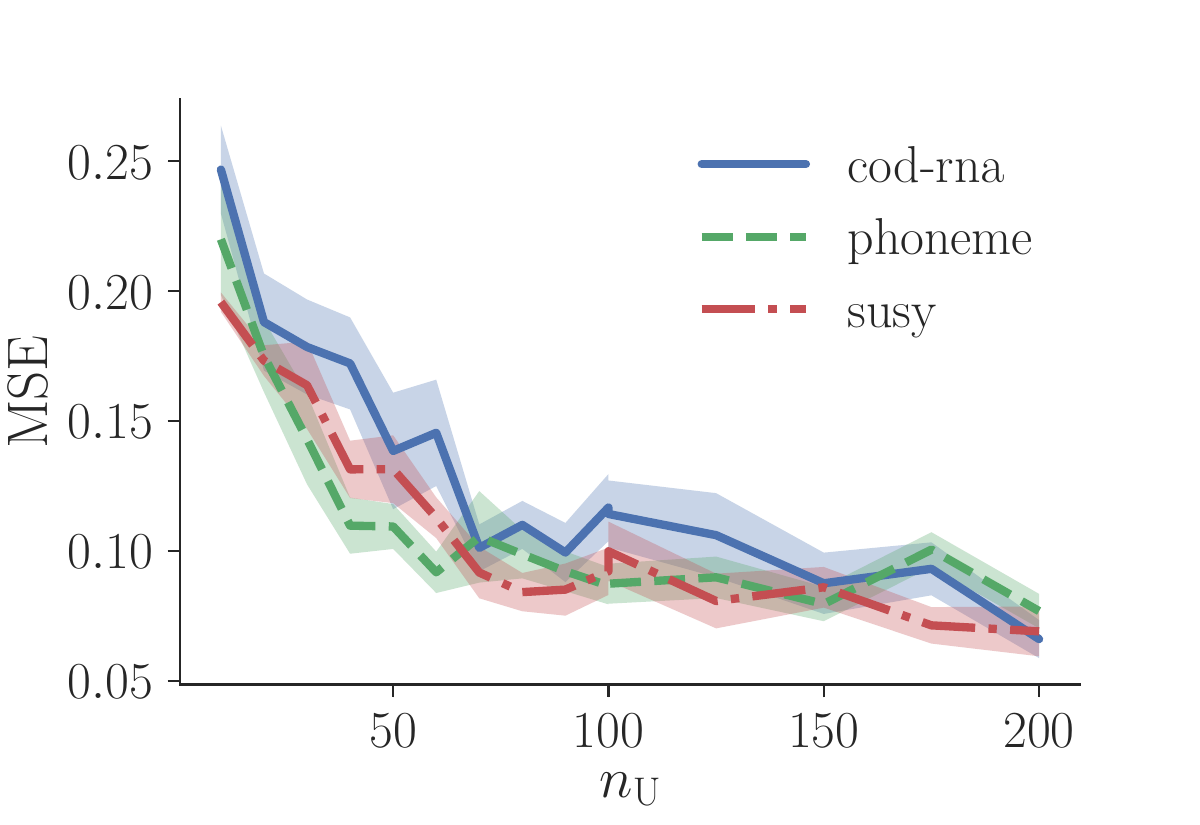}
    }%
	\caption{Average and standard error of the squared estimation error 
	of PU-SMI over $50$ trials.
	(a) $\np$ is increased while $\nun=400$ is fixed.	
	(b) $\nun$ is increased while $\np=200$ is fixed. 
	The results show that both positive and unlabeled samples 
	contribute to improving the estimation accuracy of SMI.}
	\label{fig:est-bench}
\end{figure*}

Figure~\ref{fig:est-bench} summarizes the average and standard error of the squared 
estimation error of PU-SMI over $50$ trials.\footnote{We compute the squared error 
between the estimated PU-SMI and the supervised SMI estimator \citep{BMCBio:Suzuki+etal:2009a}
with a sufficiently large number of positive and negative samples.
}
This shows that the mean squared error decreases 
both when the number of positive samples is increased
and the number of unlabeled samples is increased.
Therefore, both positive and unlabeled data contribute to
improving the estimation accuracy of SMI,
which well agrees with our theoretical analysis in Section~\ref{sec:theory}.

\subsection{Representation Learning}
\label{sec:rep-learn}
Next, we evaluate the performance of the proposed 
representation learning method, PURL.

\begin{figure*}[t]
	\centering
	\subfigure[Data and obtained subspaces]{%
 		\includegraphics[clip, width=.48\columnwidth]{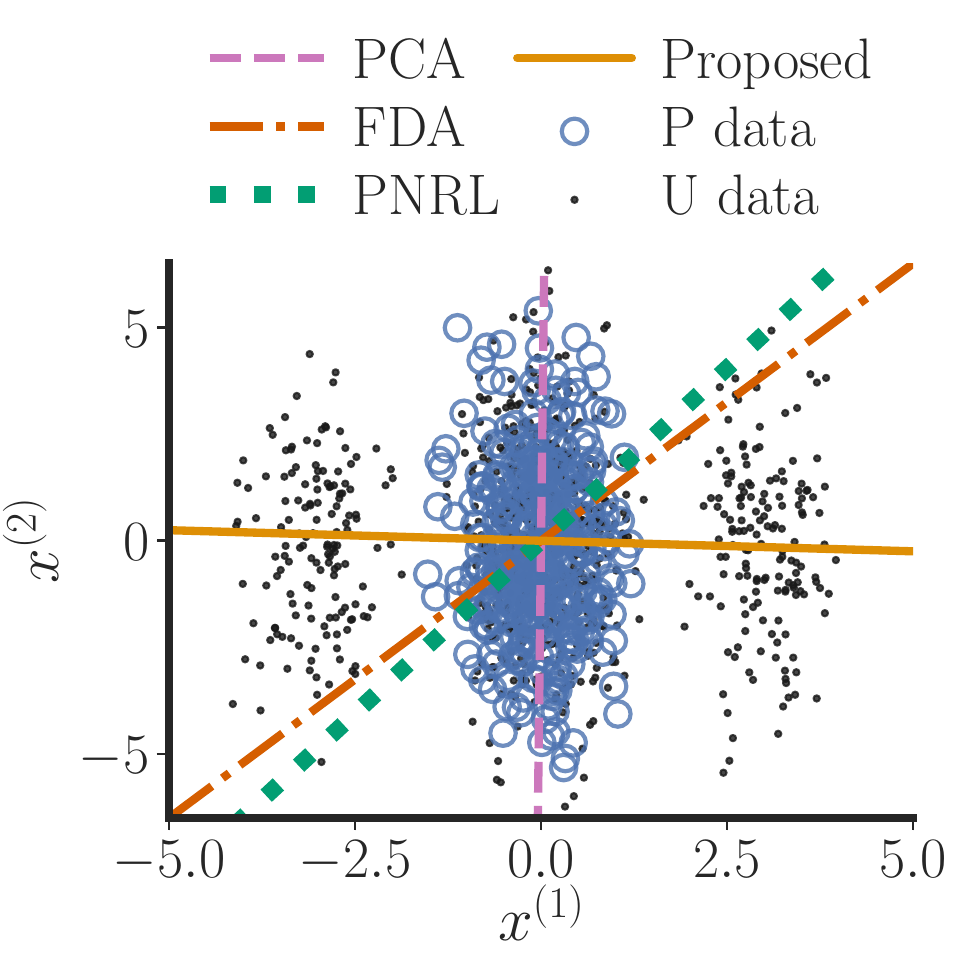}
		\label{fig:dr-data}
	}%
	\subfigure[Projected data with its labels]{%
		 \includegraphics[clip, width=.48\columnwidth]{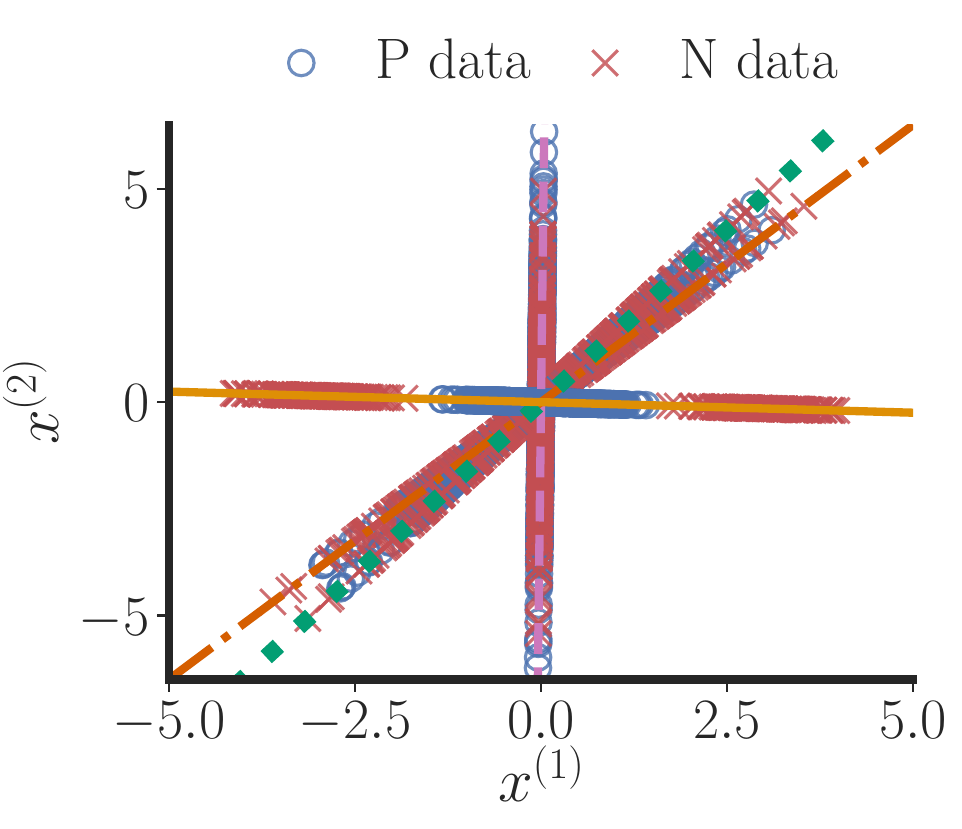}
		\label{fig:dr-proj-data}
	}%
	\caption{
	(a) Positive 
	and  
	unlabeled data. 
	The estimated subspaces obtained by 
	PCA, 
	FDA, 
	PNRL, 
	and our proposed method. 
	(b) Unlabeled data with the true labels projected onto 
	the subspaces obtained by PCA, FDA,  and our method, respectively.
	The results indicate that PCA, FDA, and PNRL smash underlying class structure,
	while the positive and negative labels are visibly separated
	in the subspace obtained by our method. 
	}
	\label{fig:dr-toy}
\end{figure*}

\paragraph{Illustration:}
We first illustrate how our proposed method works 
on an artificial data set.
We generate samples from the following densities:
\begin{align*}
p(\bx\mid y=+1)&=N\left(\bx;
\begin{pmatrix}
0 \\ 0
\end{pmatrix}, 
\begin{pmatrix}
0.25 & 0 \\ 0 & 4
\end{pmatrix}\right) , \\
p(\bx\mid y=-1)&=
\frac{1}{2}N\left(\bx;
\begin{pmatrix}
3 \\ 0
\end{pmatrix}, 
\begin{pmatrix}
0.25 & 0 \\ 0 & 4
\end{pmatrix}\right) 
+
\frac{1}{2}N\left(\bx;
\begin{pmatrix}
-3 \\ 0
\end{pmatrix}, 
\begin{pmatrix}
0.25 & 0 \\ 0 & 4
\end{pmatrix}\right) ,
\end{align*}
where $N(\bx;\bmu, \bSigma)$ is the normal density
with the mean vector $\bmu$ and the covariance matrix $\bSigma$. 
The class-prior is set at $\thetap=0.7$. 
From the densities, we draw $\np=400$ positive and 
$\nun=1000$ unlabeled samples. 
For comparison, we apply PCA, Fisher's discriminant analysis (FDA),
and PNRL\footnote{
The details of PNRL are described in Appendix~\ref{sec:pn-smi}. 
}  (the supervised counterpart of PURL) to the data.
As the label information for FDA and PNRL, U data is simply regarded as N data
even though U data is a mixture of P and N data.
Since PCA and FDA are linear transformations, 
we also use a linear transformation in PNRL and PURL
for this numerical illustration. 
Specifically, we use a two-layer perceptron for $w$.
The first fully-connected layer is used as linear transformation 
to obtain one-dimensional representation.
The rectified linear unit (ReLU) \citep{AISTATS:Glorot+etal:2011} 
is used for activation functions of the output of the first layer,
which can be seen as feature mapping functions
in the linear-in-parameter model.
The second layer is just a single connection that weighs the output of
the first layer.

We plot the subspaces obtained by PCA, FDA, PNRL, and our proposed method
in Figure~\ref{fig:dr-data}.
Since the data is distributed vertically,
the subspace obtained by PCA is almost parallel to the vertical axis (the dashed line).
FDA and PNRL return diagonal lines (the dashdot and dotted lines),
showing that regarding U data as N data is not an appropriate way.
On the other hand, the subspace obtained by our method 
is almost parallel to the horizontal axis (the solid line).
Figure~\ref{fig:dr-proj-data} plots projected labeled data onto those subspaces.
This shows that the labels of the data projected by PCA, FDA, and PNRL are hardly
distinguishable due to significant overlap,
which makes class-prior estimation very hard.
In contrast, we can easily separate the classes of 
samples projected by the proposed method,
which eases class-prior estimation.

\begin{table*}[tp] 
	\centering 
	\caption{
	Average absolute error (with standard error) between the estimated class-prior
	and the true value on benchmark datasets over $20$ trials.
	None means that the class-prior is estimated without dimension reduction methods,	  	 
	PCA is the principal component analysis,
	FDA is Fisher's discriminant analysis,
	and PNRL is the supervised counterpart of the proposed method.
	The class-prior is estimated by the method based on kernel mean embedding.
	The boldface denotes the best and comparable approaches in terms of the average
	absolute error according to the t-test at the significance level $5\%$.	
	}
	\label{tab:dr-prior-ret}	
	\resizebox{\textwidth}{!}{%
	\begin{tabular}{lccccccccc}
		\toprule 
		\multirow{2}{*}{Dataset} 		
		& \multirow{2}{*}{$\theta_\mathrm{P}$} 
		& \multirow{2}{*}{None} 		
		& \multicolumn{3}{c}{PCA} 		
		& \multirow{2}{*}{FDA}
		& \multirow{2}{*}{PNRL} 
		& \multirow{2}{*}{PURL} \\ 
		\cmidrule(lr){4-6} 
		& & & $\lfloor d/4 \rfloor$ & $\lfloor d/2 \rfloor$ 
		& $\lfloor 3d/4 \rfloor$ & 
		\\
		\midrule 
		\multirow{3}{*}{ijcnn1} 
		& $0.3$
		& $0.23$ ($0.02$) %%%% no-dimred-mlp-100 
		& $0.26$ ($0.11$) %%%% pca-div4-mlp-100 
		& $0.26$ ($0.11$) %%%% pca-div2-mlp-100 
		& $0.28$ ($0.04$) %%%% pca-3div4-mlp-100 
		& $\mathbf{0.03}$ ($\mathbf{0.01}$) %%%% lda-div4-mlp-200 
		& $0.26$ ($0.08$) %%%% pnsmi-deep-model5-mlp-200 
		& $0.21$ ($0.07$) %%%% deep-model5-mpe-all 
		\\
		& $0.5$
		& $0.18$ ($0.05$) %%%% no-dimred-mlp-100 
		& $0.14$ ($0.09$) %%%% pca-div4-mlp-100 
		& $0.14$ ($0.09$) %%%% pca-div2-mlp-100 
		& $0.17$ ($0.06$) %%%% pca-3div4-mlp-100 
		& $\mathbf{0.04}$ ($\mathbf{0.01}$) %%%% lda-div4-mlp-200 
		& $0.21$ ($0.08$) %%%% pnsmi-deep-model5-mlp-200 
		& $0.19$ ($0.07$) %%%% deep-model5-mpe-all 
		\\
		& $0.7$
		& $\mathbf{0.08}$ ($\mathbf{0.01}$) %%%% no-dimred-mlp-100 
		& $0.11$ ($0.05$) %%%% pca-div4-mlp-100 
		& $0.11$ ($0.05$) %%%% pca-div2-mlp-100 
		& $0.10$ ($0.04$) %%%% pca-3div4-mlp-100 
		& $\mathbf{0.07}$ ($\mathbf{0.01}$) %%%% lda-div4-mlp-200 
		& $0.11$ ($0.05$) %%%% pnsmi-deep-model5-mlp-200 
		& $\mathbf{0.10}$ ($\mathbf{0.01}$) %%%% deep-model5-mpe-all 
		\\
		\cmidrule(lr){1-9}
		\multirow{3}{*}{phishing} 
		& $0.3$
		& $\mathbf{0.02}$ ($\mathbf{0.00}$) %%%% no-dimred-mlp-100 
		& $\mathbf{0.02}$ ($\mathbf{0.00}$) %%%% pca-div4-mlp-100 
		& $\mathbf{0.02}$ ($\mathbf{0.00}$) %%%% pca-div2-mlp-100 
		& $\mathbf{0.02}$ ($\mathbf{0.00}$) %%%% pca-3div4-mlp-100 
		& $0.04$ ($0.02$) %%%% lda-div4-mlp-200 
		& $0.03$ ($0.01$) %%%% pnsmi-deep-model5-mlp-200 
		& $\mathbf{0.02}$ ($\mathbf{0.00}$) %%%% deep-model5-mpe-all 
		\\
		& $0.5$
		& $\mathbf{0.01}$ ($\mathbf{0.00}$) %%%% no-dimred-mlp-100 
		& $\mathbf{0.01}$ ($\mathbf{0.00}$) %%%% pca-div4-mlp-100 
		& $\mathbf{0.01}$ ($\mathbf{0.00}$) %%%% pca-div2-mlp-100 
		& $\mathbf{0.01}$ ($\mathbf{0.00}$) %%%% pca-3div4-mlp-100 
		& $0.07$ ($0.03$) %%%% lda-div4-mlp-200 
		& $0.04$ ($0.02$) %%%% pnsmi-deep-model5-mlp-200 
		& $0.03$ ($0.02$) %%%% deep-model5-mpe-all 
		\\
		& $0.7$
		& $\mathbf{0.02}$ ($\mathbf{0.00}$) %%%% no-dimred-mlp-100 
		& $\mathbf{0.02}$ ($\mathbf{0.00}$) %%%% pca-div4-mlp-100 
		& $\mathbf{0.02}$ ($\mathbf{0.00}$) %%%% pca-div2-mlp-100 
		& $\mathbf{0.02}$ ($\mathbf{0.00}$) %%%% pca-3div4-mlp-100 
		& $0.11$ ($0.04$) %%%% lda-div4-mlp-200 
		& $0.05$ ($0.03$) %%%% pnsmi-deep-model5-mlp-200 
		& $\mathbf{0.02}$ ($\mathbf{0.00}$) %%%% deep-model5-mpe-all 
		\\
		\cmidrule(lr){1-9}
		\multirow{3}{*}{mushrooms} 
		& $0.3$
		& $0.05$ ($0.01$) %%%% no-dimred-mlp-200 
		& $0.05$ ($0.01$) %%%% pca-div4-mlp-200 
		& $0.05$ ($0.01$) %%%% pca-div2-mlp-200 
		& $0.05$ ($0.01$) %%%% pca-3div4-mlp-200 
		& $0.09$ ($0.03$) %%%% lda-div4-mlp-200 
		& $\mathbf{0.03}$ ($\mathbf{0.00}$) %%%% pnsmi-deep-model5-mlp-200 
		& $\mathbf{0.03}$ ($\mathbf{0.00}$) %%%% deep-model5-mpe-all 
		\\
		& $0.5$
		& $\mathbf{0.05}$ ($\mathbf{0.01}$) %%%% no-dimred-mlp-200 
		& $\mathbf{0.05}$ ($\mathbf{0.01}$) %%%% pca-div4-mlp-200 
		& $\mathbf{0.05}$ ($\mathbf{0.01}$) %%%% pca-div2-mlp-200 
		& $\mathbf{0.05}$ ($\mathbf{0.01}$) %%%% pca-3div4-mlp-200 
		& $0.16$ ($0.03$) %%%% lda-div4-mlp-200 
		& $\mathbf{0.04}$ ($\mathbf{0.01}$) %%%% pnsmi-deep-model5-mlp-200 
		& $\mathbf{0.04}$ ($\mathbf{0.00}$) %%%% deep-model5-mpe-all 
		\\
		& $0.7$
		& $\mathbf{0.03}$ ($\mathbf{0.01}$) %%%% no-dimred-mlp-200 
		& $\mathbf{0.03}$ ($\mathbf{0.00}$) %%%% pca-div4-mlp-200 
		& $\mathbf{0.03}$ ($\mathbf{0.00}$) %%%% pca-div2-mlp-200 
		& $\mathbf{0.03}$ ($\mathbf{0.01}$) %%%% pca-3div4-mlp-200 
		& $0.20$ ($0.06$) %%%% lda-div4-mlp-200 
		& $\mathbf{0.03}$ ($\mathbf{0.00}$) %%%% pnsmi-deep-model5-mlp-200 
		& $0.04$ ($0.03$) %%%% deep-model5-mpe-all 
		\\
		\cmidrule(lr){1-9}
		\multirow{3}{*}{a9a} 
		& $0.3$
		& $0.11$ ($0.02$) %%%% no-dimred-mlp-200 
		& $0.11$ ($0.02$) %%%% pca-div4-mlp-200 
		& $0.11$ ($0.02$) %%%% pca-div2-mlp-200 
		& $0.11$ ($0.02$) %%%% pca-3div4-mlp-200 
		& $\mathbf{0.05}$ ($\mathbf{0.01}$) %%%% lda-div4-mlp-200 
		& $0.08$ ($0.03$) %%%% pnsmi-deep-model5-mlp-200 
		& $\mathbf{0.04}$ ($\mathbf{0.00}$) %%%% deep-model5-mpe-all 
		\\
		& $0.5$
		& $0.10$ ($0.02$) %%%% no-dimred-mlp-200 
		& $0.10$ ($0.02$) %%%% pca-div4-mlp-200 
		& $0.10$ ($0.02$) %%%% pca-div2-mlp-200 
		& $0.10$ ($0.02$) %%%% pca-3div4-mlp-200 
		& $0.09$ ($0.04$) %%%% lda-div4-mlp-200 
		& $0.09$ ($0.03$) %%%% pnsmi-deep-model5-mlp-200 
		& $\mathbf{0.04}$ ($\mathbf{0.01}$) %%%% deep-model5-mpe-all 
		\\
		& $0.7$
		& $0.08$ ($0.03$) %%%% no-dimred-mlp-200 
		& $0.08$ ($0.03$) %%%% pca-div4-mlp-200 
		& $0.08$ ($0.03$) %%%% pca-div2-mlp-200 
		& $0.08$ ($0.03$) %%%% pca-3div4-mlp-200 
		& $0.18$ ($0.06$) %%%% lda-div4-mlp-200 
		& $0.08$ ($0.03$) %%%% pnsmi-deep-model5-mlp-200 
		& $\mathbf{0.04}$ ($\mathbf{0.01}$) %%%% deep-model5-mpe-all 
		\\
		\cmidrule(lr){1-9}
		\multirow{3}{*}{MNIST} 
		& $0.3$
		& $0.09$ ($0.02$) %%%% no-dimred-mlp-200 
		& $0.09$ ($0.02$) %%%% pca-div4-mlp-200 
		& $0.09$ ($0.02$) %%%% pca-div2-mlp-200 
		& $0.09$ ($0.02$) %%%% pca-3div4-mlp-200 
		& $0.27$ ($0.01$) %%%% lda-div4-mlp-200 
		& $\mathbf{0.01}$ ($\mathbf{0.00}$) %%%% pnsmi-deep-model5-mlp-200 
		& $0.05$ ($0.02$) %%%% deep-model5-mpe-all 
		\\
		& $0.5$
		& $0.15$ ($0.11$) %%%% no-dimred-mlp-200 
		& $0.15$ ($0.11$) %%%% pca-div4-mlp-200 
		& $0.15$ ($0.11$) %%%% pca-div2-mlp-200 
		& $0.15$ ($0.11$) %%%% pca-3div4-mlp-200 
		& $0.46$ ($0.01$) %%%% lda-div4-mlp-200 
		& $\mathbf{0.03}$ ($\mathbf{0.00}$) %%%% pnsmi-deep-model5-mlp-200 
		& $0.06$ ($0.03$) %%%% deep-model5-mpe-all 
		\\
		& $0.7$
		& $0.60$ ($0.21$) %%%% no-dimred-mlp-200 
		& $0.60$ ($0.21$) %%%% pca-div4-mlp-200 
		& $0.60$ ($0.21$) %%%% pca-div2-mlp-200 
		& $0.60$ ($0.21$) %%%% pca-3div4-mlp-200 
		& $0.65$ ($0.02$) %%%% lda-div4-mlp-200 
		& $\mathbf{0.06}$ ($\mathbf{0.01}$) %%%% pnsmi-deep-model5-mlp-200 
		& $\mathbf{0.07}$ ($\mathbf{0.01}$) %%%% deep-model5-mpe-all 
		\\
		\cmidrule(lr){1-9}
		\multirow{3}{*}{F-MNIST} 
		& $0.3$
		& $\mathbf{0.02}$ ($\mathbf{0.00}$) %%%% no-dimred-mlp-200 
		& $\mathbf{0.02}$ ($\mathbf{0.00}$) %%%% pca-div4-mlp-200 
		& $\mathbf{0.02}$ ($\mathbf{0.00}$) %%%% pca-div2-mlp-200 
		& $\mathbf{0.02}$ ($\mathbf{0.00}$) %%%% pca-3div4-mlp-200 
		& $0.25$ ($0.01$) %%%% lda-div4-mlp-200 
		& $\mathbf{0.03}$ ($\mathbf{0.00}$) %%%% pnsmi-deep-model5-mlp-200 
		& $\mathbf{0.03}$ ($\mathbf{0.00}$) %%%% deep-model5-mpe-all 
		\\
		& $0.5$
		& $\mathbf{0.03}$ ($\mathbf{0.00}$) %%%% no-dimred-mlp-200 
		& $\mathbf{0.03}$ ($\mathbf{0.00}$) %%%% pca-div4-mlp-200 
		& $\mathbf{0.03}$ ($\mathbf{0.00}$) %%%% pca-div2-mlp-200 
		& $\mathbf{0.03}$ ($\mathbf{0.00}$) %%%% pca-3div4-mlp-200 
		& $0.45$ ($0.01$) %%%% lda-div4-mlp-200 
		& $\mathbf{0.02}$ ($\mathbf{0.00}$) %%%% pnsmi-deep-model5-mlp-200 
		& $0.04$ ($0.03$) %%%% deep-model5-mpe-all 
		\\
		& $0.7$
		& $\mathbf{0.03}$ ($\mathbf{0.00}$) %%%% no-dimred-mlp-200 
		& $\mathbf{0.03}$ ($\mathbf{0.00}$) %%%% pca-div4-mlp-200 
		& $\mathbf{0.03}$ ($\mathbf{0.00}$) %%%% pca-div2-mlp-200 
		& $\mathbf{0.03}$ ($\mathbf{0.00}$) %%%% pca-3div4-mlp-200 
		& $0.66$ ($0.02$) %%%% lda-div4-mlp-200 
		& $\mathbf{0.02}$ ($\mathbf{0.00}$) %%%% pnsmi-deep-model5-mlp-200 
		& $0.07$ ($0.03$) %%%% deep-model5-mpe-all 
		\\
		\cmidrule(lr){1-9}
		\multirow{3}{*}{20 News} 
		& $0.3$
		& $\mathbf{0.04}$ ($\mathbf{0.00}$) %%%% no-dimred 
		& $\mathbf{0.04}$ ($\mathbf{0.01}$) %%%% pca-div4 
		& $\mathbf{0.04}$ ($\mathbf{0.00}$) %%%% pca-div2 
		& $\mathbf{0.04}$ ($\mathbf{0.00}$) %%%% pca-3div4 
		& $0.29$ ($0.00$) %%%% lda-div4 
		& $0.29$ ($0.09$) %%%% pnsmi-deep-model7-mlp-200 
		& $\mathbf{0.03}$ ($\mathbf{0.01}$) %%%% deep-model7-mpe-all 
		\\
		& $0.5$
		& $0.08$ ($0.03$) %%%% no-dimred 
		& $\mathbf{0.06}$ ($\mathbf{0.01}$) %%%% pca-div4 
		& $\mathbf{0.07}$ ($\mathbf{0.01}$) %%%% pca-div2 
		& $0.08$ ($0.03$) %%%% pca-3div4 
		& $0.49$ ($0.00$) %%%% lda-div4 
		& $0.25$ ($0.07$) %%%% pnsmi-deep-model7-mlp-200 
		& $\mathbf{0.05}$ ($\mathbf{0.01}$) %%%% deep-model7-mpe-all 
		\\
		& $0.7$
		& $0.69$ ($0.00$) %%%% no-dimred 
		& $0.69$ ($0.00$) %%%% pca-div4 
		& $0.69$ ($0.00$) %%%% pca-div2 
		& $0.69$ ($0.00$) %%%% pca-3div4 
		& $0.69$ ($0.00$) %%%% lda-div4 
		& $\mathbf{0.13}$ ($\mathbf{0.03}$) %%%% pnsmi-deep-model7-mlp-200 
		& $\mathbf{0.07}$ ($\mathbf{0.01}$) %%%% deep-model7-mpe-all 
		\\		
		\bottomrule 
	\end{tabular}
	}
\end{table*}
\begin{table*}[tp] 
	\centering	
	\caption{
	Average misclassification rates (with standard error) 
	on benchmark datasets over $20$ trials.
	The boldface denotes the best and comparable approaches
	in terms of the average absolute error according to
	the t-test at the significance level $5\%$.
	}
	\label{tab:dr-clfy-ret}
	\resizebox{\textwidth}{!}{%
	\begin{tabular}{lccccccccc}
		\toprule 
		\multirow{2}{*}{Dataset} & \multirow{2}{*}{$\theta_\mathrm{P}$} 
		& \multirow{2}{*}{None} 
		& \multicolumn{3}{c}{PCA} 
		& \multirow{2}{*}{FDA} 
		& \multirow{2}{*}{PNRL} 
		& \multirow{2}{*}{PURL} \\ 
		\cmidrule(lr){4-6} 
		& & & $\lfloor d/4 \rfloor$ & $\lfloor d/2 \rfloor$ 
		& $\lfloor 3d/4 \rfloor$ & 
		\\
		\midrule 
		\multirow{3}{*}{ijcnn1} 
		& $0.3$
		& $25.32$ ($1.23$) %%%% no-dimred-mlp-100 
		& $27.79$ ($2.05$) %%%% pca-div4-mlp-100 
		& $27.79$ ($2.05$) %%%% pca-div2-mlp-100 
		& $29.21$ ($1.30$) %%%% pca-3div4-mlp-100 
		& $\mathbf{7.00}$ ($\mathbf{0.59}$) %%%% lda-div4-mlp-200 
		& $28.57$ ($2.52$) %%%% pnsmi-deep-model5-mlp-200 
		& $25.92$ ($2.64$) %%%% deep-model5-mpe-all 
		\\
		& $0.5$
		& $21.43$ ($1.22$) %%%% no-dimred-mlp-100 
		& $17.88$ ($1.72$) %%%% pca-div4-mlp-100 
		& $17.88$ ($1.72$) %%%% pca-div2-mlp-100 
		& $20.75$ ($1.15$) %%%% pca-3div4-mlp-100 
		& $\mathbf{8.25}$ ($\mathbf{1.33}$) %%%% lda-div4-mlp-200 
		& $26.52$ ($2.26$) %%%% pnsmi-deep-model5-mlp-200 
		& $21.52$ ($1.69$) %%%% deep-model5-mpe-all 
		\\
		& $0.7$
		& $\mathbf{12.07}$ ($\mathbf{0.53}$) %%%% no-dimred-mlp-100 
		& $14.94$ ($1.02$) %%%% pca-div4-mlp-100 
		& $14.94$ ($1.02$) %%%% pca-div2-mlp-100 
		& $\mathbf{14.61}$ ($\mathbf{1.15}$) %%%% pca-3div4-mlp-100 
		& $\mathbf{11.23}$ ($\mathbf{1.34}$) %%%% lda-div4-mlp-200 
		& $17.34$ ($1.58$) %%%% pnsmi-deep-model5-mlp-200 
		& $\mathbf{13.70}$ ($\mathbf{1.04}$) %%%% deep-model5-mpe-all 
		\\
		\cmidrule(lr){1-9}
		\multirow{3}{*}{phishing} 
		& $0.3$
		& $\mathbf{7.41}$ ($\mathbf{0.46}$) %%%% no-dimred-mlp-100 
		& $\mathbf{7.46}$ ($\mathbf{0.48}$) %%%% pca-div4-mlp-100 
		& $\mathbf{7.46}$ ($\mathbf{0.48}$) %%%% pca-div2-mlp-100 
		& $\mathbf{7.57}$ ($\mathbf{0.46}$) %%%% pca-3div4-mlp-100 
		& $\mathbf{10.30}$ ($\mathbf{2.42}$) %%%% lda-div4-mlp-200 
		& $11.09$ ($0.98$) %%%% pnsmi-deep-model5-mlp-200 
		& $\mathbf{7.62}$ ($\mathbf{0.45}$) %%%% deep-model5-mpe-all 
		\\
		& $0.5$
		& $\mathbf{12.85}$ ($\mathbf{2.11}$) %%%% no-dimred-mlp-100 
		& $\mathbf{9.75}$ ($\mathbf{0.46}$) %%%% pca-div4-mlp-100 
		& $\mathbf{9.75}$ ($\mathbf{0.46}$) %%%% pca-div2-mlp-100 
		& $\mathbf{9.82}$ ($\mathbf{0.40}$) %%%% pca-3div4-mlp-100 
		& $24.43$ ($3.09$) %%%% lda-div4-mlp-200 
		& $32.02$ ($3.05$) %%%% pnsmi-deep-model5-mlp-200 
		& $\mathbf{10.05}$ ($\mathbf{0.46}$) %%%% deep-model5-mpe-all 
		\\
		& $0.7$
		& $\mathbf{8.07}$ ($\mathbf{0.44}$) %%%% no-dimred-mlp-100 
		& $\mathbf{8.85}$ ($\mathbf{1.08}$) %%%% pca-div4-mlp-100 
		& $\mathbf{8.85}$ ($\mathbf{1.08}$) %%%% pca-div2-mlp-100 
		& $\mathbf{7.63}$ ($\mathbf{0.37}$) %%%% pca-3div4-mlp-100 
		& $25.62$ ($1.40$) %%%% lda-div4-mlp-200 
		& $29.04$ ($0.73$) %%%% pnsmi-deep-model5-mlp-200 
		& $\mathbf{8.02}$ ($\mathbf{0.37}$) %%%% deep-model5-mpe-all 
		\\
		\cmidrule(lr){1-9}
		\multirow{3}{*}{mushrooms} 
		& $0.3$
		& $0.73$ ($0.20$) %%%% no-dimred-mlp-200 
		& $\mathbf{1.15}$ ($\mathbf{0.58}$) %%%% pca-div4-mlp-200 
		& $0.57$ ($0.14$) %%%% pca-div2-mlp-200 
		& $\mathbf{0.49}$ ($\mathbf{0.14}$) %%%% pca-3div4-mlp-200 
		& $1.52$ ($0.36$) %%%% lda-div4-mlp-200 
		& $\mathbf{0.24}$ ($\mathbf{0.06}$) %%%% pnsmi-deep-model5-mlp-200 
		& $\mathbf{0.43}$ ($\mathbf{0.09}$) %%%% deep-model5-mpe-all 
		\\
		& $0.5$
		& $\mathbf{0.57}$ ($\mathbf{0.11}$) %%%% no-dimred-mlp-200 
		& $\mathbf{0.57}$ ($\mathbf{0.11}$) %%%% pca-div4-mlp-200 
		& $\mathbf{0.78}$ ($\mathbf{0.16}$) %%%% pca-div2-mlp-200 
		& $\mathbf{0.57}$ ($\mathbf{0.11}$) %%%% pca-3div4-mlp-200 
		& $3.40$ ($0.47$) %%%% lda-div4-mlp-200 
		& $\mathbf{1.10}$ ($\mathbf{0.24}$) %%%% pnsmi-deep-model5-mlp-200 
		& $\mathbf{3.40}$ ($\mathbf{2.39}$) %%%% deep-model5-mpe-all 
		\\
		& $0.7$
		& $\mathbf{1.42}$ ($\mathbf{0.28}$) %%%% no-dimred-mlp-200 
		& $\mathbf{1.42}$ ($\mathbf{0.28}$) %%%% pca-div4-mlp-200 
		& $\mathbf{1.50}$ ($\mathbf{0.27}$) %%%% pca-div2-mlp-200 
		& $\mathbf{1.42}$ ($\mathbf{0.28}$) %%%% pca-3div4-mlp-200 
		& $6.38$ ($0.66$) %%%% lda-div4-mlp-200 
		& $\mathbf{1.40}$ ($\mathbf{0.27}$) %%%% pnsmi-deep-model5-mlp-200 
		& $\mathbf{1.61}$ ($\mathbf{0.48}$) %%%% deep-model5-mpe-all 
		\\
		\cmidrule(lr){1-9}
		\multirow{3}{*}{a9a} 
		& $0.3$
		& $24.93$ ($1.19$) %%%% no-dimred-mlp-200 
		& $26.49$ ($1.89$) %%%% pca-div4-mlp-200 
		& $26.49$ ($1.89$) %%%% pca-div2-mlp-200 
		& $26.20$ ($1.73$) %%%% pca-3div4-mlp-200 
		& $\mathbf{21.09}$ ($\mathbf{0.59}$) %%%% lda-div4-mlp-200 
		& $26.31$ ($2.36$) %%%% pnsmi-deep-model5-mlp-200 
		& $\mathbf{22.32}$ ($\mathbf{0.65}$) %%%% deep-model5-mpe-all 
		\\
		& $0.5$
		& $30.35$ ($1.55$) %%%% no-dimred-mlp-200 
		& $26.07$ ($1.01$) %%%% pca-div4-mlp-200 
		& $26.07$ ($1.01$) %%%% pca-div2-mlp-200 
		& $29.52$ ($1.81$) %%%% pca-3div4-mlp-200 
		& $\mathbf{22.70}$ ($\mathbf{0.77}$) %%%% lda-div4-mlp-200 
		& $27.48$ ($1.47$) %%%% pnsmi-deep-model5-mlp-200 
		& $\mathbf{23.70}$ ($\mathbf{0.67}$) %%%% deep-model5-mpe-all 
		\\
		& $0.7$
		& $\mathbf{20.35}$ ($\mathbf{0.80}$) %%%% no-dimred-mlp-200 
		& $\mathbf{20.54}$ ($\mathbf{0.61}$) %%%% pca-div4-mlp-200 
		& $\mathbf{20.54}$ ($\mathbf{0.61}$) %%%% pca-div2-mlp-200 
		& $\mathbf{19.94}$ ($\mathbf{0.78}$) %%%% pca-3div4-mlp-200 
		& $\mathbf{19.70}$ ($\mathbf{0.97}$) %%%% lda-div4-mlp-200 
		& $\mathbf{20.59}$ ($\mathbf{0.60}$) %%%% pnsmi-deep-model5-mlp-200 
		& $\mathbf{19.39}$ ($\mathbf{0.66}$) %%%% deep-model5-mpe-all 
		\\
		\cmidrule(lr){1-9}
		\multirow{3}{*}{MNIST} 
		& $0.3$
		& $24.58$ ($2.82$) %%%% no-dimred-mlp-200 
		& $17.99$ ($1.44$) %%%% pca-div4-mlp-200 
		& $17.99$ ($1.44$) %%%% pca-div2-mlp-200 
		& $22.18$ ($2.75$) %%%% pca-3div4-mlp-200 
		& $20.92$ ($0.74$) %%%% lda-div4-mlp-200 
		& $\mathbf{12.74}$ ($\mathbf{0.63}$) %%%% pnsmi-deep-model5-mlp-200 
		& $\mathbf{11.76}$ ($\mathbf{0.78}$) %%%% deep-model5-mpe-all 
		\\
		& $0.5$
		& $23.00$ ($1.60$) %%%% no-dimred-mlp-200 
		& $22.35$ ($1.10$) %%%% pca-div4-mlp-200 
		& $22.35$ ($1.10$) %%%% pca-div2-mlp-200 
		& $23.55$ ($1.80$) %%%% pca-3div4-mlp-200 
		& $42.10$ ($1.85$) %%%% lda-div4-mlp-200 
		& $\mathbf{15.35}$ ($\mathbf{0.75}$) %%%% pnsmi-deep-model5-mlp-200 
		& $\mathbf{18.18}$ ($\mathbf{2.43}$) %%%% deep-model5-mpe-all 
		\\
		& $0.7$
		& $53.34$ ($3.78$) %%%% no-dimred-mlp-200 
		& $52.19$ ($4.41$) %%%% pca-div4-mlp-200 
		& $54.42$ ($3.99$) %%%% pca-div2-mlp-200 
		& $53.39$ ($3.74$) %%%% pca-3div4-mlp-200 
		& $60.86$ ($1.25$) %%%% lda-div4-mlp-200 
		& $\mathbf{16.38}$ ($\mathbf{0.84}$) %%%% pnsmi-deep-model5-mlp-200 
		& $\mathbf{18.64}$ ($\mathbf{2.83}$) %%%% deep-model5-mpe-all 
		\\
		\cmidrule(lr){1-9}
		\multirow{3}{*}{F-MNIST} 
		& $0.3$
		& $\mathbf{14.88}$ ($\mathbf{1.30}$) %%%% no-dimred-mlp-200 
		& $\mathbf{18.02}$ ($\mathbf{2.86}$) %%%% pca-div4-mlp-200 
		& $\mathbf{18.02}$ ($\mathbf{2.86}$) %%%% pca-div2-mlp-200 
		& $\mathbf{15.12}$ ($\mathbf{1.18}$) %%%% pca-3div4-mlp-200 
		& $19.24$ ($0.91$) %%%% lda-div4-mlp-200 
		& $\mathbf{14.54}$ ($\mathbf{1.14}$) %%%% pnsmi-deep-model5-mlp-200 
		& $\mathbf{13.54}$ ($\mathbf{0.75}$) %%%% deep-model5-mpe-all 
		\\
		& $0.5$
		& $\mathbf{13.40}$ ($\mathbf{0.69}$) %%%% no-dimred-mlp-200 
		& $\mathbf{12.05}$ ($\mathbf{0.96}$) %%%% pca-div4-mlp-200 
		& $\mathbf{12.05}$ ($\mathbf{0.96}$) %%%% pca-div2-mlp-200 
		& $\mathbf{13.22}$ ($\mathbf{0.62}$) %%%% pca-3div4-mlp-200 
		& $37.73$ ($1.56$) %%%% lda-div4-mlp-200 
		& $\mathbf{12.15}$ ($\mathbf{0.48}$) %%%% pnsmi-deep-model5-mlp-200 
		& $\mathbf{14.10}$ ($\mathbf{1.16}$) %%%% deep-model5-mpe-all 
		\\
		& $0.7$
		& $\mathbf{9.94}$ ($\mathbf{1.30}$) %%%% no-dimred-mlp-200 
		& $\mathbf{8.89}$ ($\mathbf{0.84}$) %%%% pca-div4-mlp-200 
		& $\mathbf{8.89}$ ($\mathbf{0.84}$) %%%% pca-div2-mlp-200 
		& $\mathbf{8.54}$ ($\mathbf{0.84}$) %%%% pca-3div4-mlp-200 
		& $55.65$ ($2.14$) %%%% lda-div4-mlp-200 
		& $\mathbf{8.65}$ ($\mathbf{0.83}$) %%%% pnsmi-deep-model5-mlp-200 
		& $\mathbf{9.29}$ ($\mathbf{0.47}$) %%%% deep-model5-mpe-all 
		\\
		\cmidrule(lr){1-9}
		\multirow{3}{*}{20 News} 
		& $0.3$
		& $38.89$ ($3.00$) %%%% no-dimred 
		& $40.30$ ($3.64$) %%%% pca-div4 
		& $42.48$ ($3.54$) %%%% pca-div2 
		& $38.70$ ($3.81$) %%%% pca-3div4 
		& $\mathbf{18.66}$ ($\mathbf{0.47}$) %%%% lda-div4 
		& $66.62$ ($1.59$) %%%% pnsmi-deep-model7-mlp-200 
		& $36.31$ ($4.13$) %%%% deep-model7-mpe-all 
		\\
		& $0.5$
		& $44.48$ ($1.82$) %%%% no-dimred 
		& $43.85$ ($2.03$) %%%% pca-div4 
		& $46.67$ ($1.15$) %%%% pca-div2 
		& $47.77$ ($0.87$) %%%% pca-3div4 
		& $\mathbf{34.73}$ ($\mathbf{0.80}$) %%%% lda-div4 
		& $50.00$ ($0.00$) %%%% pnsmi-deep-model7-mlp-200 
		& $45.88$ ($1.64$) %%%% deep-model7-mpe-all 
		\\
		& $0.7$
		& $50.69$ ($0.95$) %%%% no-dimred 
		& $53.61$ ($0.73$) %%%% pca-div4 
		& $51.77$ ($1.05$) %%%% pca-div2 
		& $50.36$ ($0.83$) %%%% pca-3div4 
		& $50.69$ ($0.95$) %%%% lda-div4 
		& $\mathbf{30.61}$ ($\mathbf{0.59}$) %%%% pnsmi-deep-model7-mlp-200 
		& $\mathbf{29.85}$ ($\mathbf{0.13}$) %%%% deep-model7-mpe-all 
		\\		
		\bottomrule 
	\end{tabular}	
	}%
\end{table*}

\paragraph{Benchmark Data:}
Next we apply the PURL method to benchmark datasets.
To obtain low-dimensional representation, 
we set $m=20$ and use a fully-connected neural network 
with four layers ($d$-$60$-$20$-$1$; $\bv$ is $d$-$60$-$20$ and $g$ is $20$-$1$.)
for $w$ except text classification dataset.
For text classification datasets, we use another 
fully-connected neural network with four layers
($d$-$30$-$10$-$1$) for $w$, i.e., $m=10$. 
ReLU is used as activation functions for hidden layers, 
and \emph{batch normalization} \citep{ICML:Ioffe+Szegedy:2015}
is applied to all hidden layers.
Stochastic gradient descent is used for optimization with learning rate $0.001$. 
Also, weight decay with $0.0005$ and gradient noise with $0.01$ are applied.
We iteratively update $w$ with four mini-batches and $\bv$ with one mini-batch.

We compare the accuracy of class-prior estimation with and without dimension reduction.
For comparison, we also consider PCA, FDA, and PNRL.
For PCA, we vary the numbers of components as follows: $\lfloor d/4 \rfloor$, 
$\lfloor d/2 \rfloor$, and $\lfloor 3d/4 \rfloor$,
where $\lfloor \cdot \rfloor$ is the floor function.
For FDA, the reduced dimension is $1$ due to the property of FDA
\citep{BOOK:Hastie+etal:2009} in which the reduced dimension becomes
the minimum of $m$ or $(\text{the number of classes}-1)$. 
The neural network for PNRL is the same as the one for the proposed method.

As a class-prior estimation method, we use the method based on
the \emph{kernel mean embedding} (KM) method proposed by \citet{ICML:Ramaswamy+etal:2016}.
With the estimated class-prior, we then train a fully-connected neural network 
with five layers ($m$-$300$-$300$-$300$-$1$).
ReLU is used as activation functions for hidden layers, 
and batch normalization is applied to all hidden layers.
Except for text classification datasets, we train the neural networks
by Adam \citep{ICLR:Kingma+Ba:2015} until $200$ epochs.
For text classification datasets, we use AdaGrad \citep{JMLR:Duchi+etal:2011}
and set the number of epochs to $300$.
For non-negative PU learning \citep{NIPS:Kiryo+etal:2017},
we use the sigmoid loss function and set $\beta$ and $\gamma$ in the paper
to $0$ and $1$, respectively.

We use the ijcnn$1$, phishing, mushrooms, and a$9$a 
datasets taken from the LIBSVM webpage \citep{LibSVM:Chang+etal:2011}.
Also, we use the MNIST \citep{IEEE:Lecun:1998},
Fashion-MNIST (F-MNIST) \citep{FMNIST:Xiao+etal:2017},
and $20$ Newsgroups \citep{ICML:Lang:1995} datasets.
For the MNIST and F-MNIST datasets, we divide the
whole classes into $2$ groups to make binary classification tasks.
For the $20$ Newsgroups dataset, we use
the ``com'' topic as the positive class
and the ``sci'' topic as the negative class,\footnote{
See \url{http://qwone.com/~jason/20Newsgroups/} for 
the details of topics.
} 
and make $2000$-dimensional \emph{tf-idf} vector.  
From the datasets, we draw $\np=1000$ positive and $\nun=2000$ unlabeled samples.
For validation, we use $\np=50$ and $\nun=200$ samples.

Table~\ref{tab:dr-prior-ret} lists the average absolute error 
between the estimated class-prior and the true value. 
Overall, our proposed dimension reduction method tends to
outperform other methods,
meaning that our method provides useful low-dimensional representation.
Except for the ijcnn$1$ dataset,
the error of FDA tends to be larger than the other methods,
implying that regarding U data as N data does not help in class-prior estimation.   
For the mushrooms and a$9$a datasets, 
applying the unsupervised dimension reduction method, 
PCA, does not improve the estimation accuracy,
while our method reduces the error of class-prior estimation.
In particular, for the $20$ Newsgroups dataset,
the existing approaches (PCA, FDA, and PNRL) perform poorly.
In contrast, applying our method significantly reduces 
the error of class-prior estimation.
	
Then, we summarize the average misclassification rates
in Table~\ref{tab:dr-clfy-ret}.
Since the accuracy of class-prior estimation is improved
on the mushrooms and a$9$a datasets,
the classification accuracy is also improved. 
In particular, the classification results on 
the $20$~Newsgroups dataset with $\thetap=0.7$ are improved substantially.
Overall, our proposed method tends to
give the lower or comparable misclassification rates compared 
with the other methods.

\section{Conclusions}
\label{sec:conclusion}
In this paper, we proposed an information-theoretic 
representation learning method from positive and unlabeled (PU) data. 
% for learning from positive and unlabeled data (PU learning).
Our method is based on the information maximization principle,
and find low-dimensional representation maximally 
preserving a squared-loss variant of mutual information (SMI) 
between inputs and labels.
Unlike the existing PU learning methods,
since our representation learning method can be executed 
without knowing an estimate of the class-prior in advance,
our method can also be used as preprocessing
for the class-prior estimation method.
Through numerical experiments, we demonstrated 
the effectiveness of our method.

\subsection*{Acknowledgements}
TS was supported by KAKENHI $15$J$09111$.
GN was supported by the JST CREST JPMJCR$1403$.
MS was supported by KAKENHI $17$H$01760$. 
We thank Ikko Yamane, Ryuichi Kiryo, and Takeshi Teshima
for their comments.

\bibliographystyle{plainnat-reversed}
\bibliography{pu_refs}

\newpage \newpage
\appendix
\onecolumn

\appendix
\section{Proof of Theorem~\ref{theorem:PUSMI=SMI}}
\label{app:theorem:PUSMI=SMI}
\begin{proof}
Let us express SMI in Eq.~\eqref{eq:smi-def} as 
\begin{align}
  \SMI&=\frac{\thetap}{2}\int\Big(
	\frac{p(\bx\mid y=+1)}{p(\bx)}-1\Big)^2 p(\bx) \dbx 
	+\frac{\thetan}{2}\int\Big(\frac{p(\bx\mid y=-1)}{p(\bx)}-1\Big)^2
	p(\bx) \dbx	.
	\label{eq:smi-split}
\end{align}
From the marginal density $p(\bx)$, we have
\begin{align*}
  \thetan\frac{p(\bx\mid y=-1)}{p(\bx)}&=
                                         1-\thetap\frac{p(\bx\mid y=+1)}{p(\bx)} , \notag \\
  \thetan\Big(\frac{p(\bx\mid y=-1)}{p(\bx)}-1\Big) 
                                       &=\thetap\Big(1-\frac{p(\bx\mid y=+1)}{p(\bx)}\Big) , \notag \\
  \Big(\frac{p(\bx\mid y=-1)}{p(\bx)}-1\Big)^2
                                       &=\frac{\theta_\rP^2}{\theta_\rN^2}
                                         \Big(\frac{p(\bx\mid y=+1)}{p(\bx)}-1\Big)^2 ,
\end{align*}
where the equality between the first and second equations  
can be confirmed by using $\thetap+\thetan=1$.
Plugging the last equation into the second term 
of Eq.~\eqref{eq:smi-split}, we then obtain an expression of SMI
only with positive and unlabeled data (PU-SMI) as
\begin{align*}
  \SMI\!=\frac{\thetap}{2\thetan}\!
	\int\!\!\Big(\frac{p(\bx\mid y=+1)}{p(\bx)}-1\Big)^2 \! p(\bx)\dbx 
	=:\PUSMI . 
\end{align*}
\end{proof}

\section{Proof of Theorem~\ref{theorem:PUSMI-lowerbound}}
\label{app:theorem:PUSMI-lowerbound}
\begin{proof}
Let 
\begin{align*}
s(\bx):=\frac{p(\bx\mid y=+1)}{p(\bx)}
\end{align*}
be the density ratio.
Then, PU-SMI can be expressed as
\begin{align*}
\PUSMI&=\frac{\thetap}{2\thetan}\int (s(\bx)-1)^2p(\bx)\dbx \\
&=\frac{\thetap}{\thetan}\Big( \frac{1}{2}\int s^2(\bx)p(\bx)\dbx
	- \frac{1}{2}\Big) ,
\end{align*}
where $s(\bx)p(\bx)=p(\bx\mid y=+1)$ is used.
Based on the \emph{Fenchel inequality} \citep{BOOK:Boyd+Vandenberghe:2004},
for any function $f(\bx)$ in a function class $\calF$, we have 
\begin{align*}
\frac{1}{2}s^2(\bx) \geq f(\bx)s(\bx) - \frac{1}{2} f^2(\bx)  .
\end{align*}
Then, we obtain the lower bound of the PU-SMI by
\begin{align*}
\PUSMI
&\geq	\frac{\thetap}{\thetan}\Big(\int f(\bx)p(\bx\mid y=+1)\dbx 
	-\frac{1}{2}\int f^2(\bx)p(\bx)\dbx
	-\frac{1}{2}\Big)  \\
&=\frac{\thetap}{\thetan}
\Big(-J_\PU(f)-\frac{1}{2}\Big) ,
\end{align*}
where
\begin{align*}
  J_\PU(f):=\frac{1}{2}\int f^2(\bx)p(\bx)\dbx
  	-\int f(\bx)p(\bx\mid
  y=+1)\dbx .
\end{align*}
Thus, from the \emph{Fenchel duality} \citep{CRM:Keziou:2003,ISIT:Nguyen::2007},
we have
\begin{align*}
\PUSMI
&=\sup_{f\in\calF}\; \frac{\thetap}{\thetan}
\Big(-J_\PU(f)-\frac{1}{2}\Big) ,
\end{align*}
where equality in the supremum is attained when 
$f(\bx)=s(\bx)=p(\bx\mid y=+1)/p(\bx)$ and $s\in\calF$.
\end{proof}

\section{Proof of Theorem~\ref{thm:rate}}
\label{app:proof-rate}
The idea of the proof is to view the approximated squared error
as perturbed optimization of expected one.
In the analysis, we focus on the linear-in-parameter model
$w(\bx)=\sum^b_{\ell=1}\beta_\ell\phi_\ell(\bx)=\bbeta^\top\bphi(\bx)$. 
We assume that $0\leq\phi_\ell\leq 1$ for all $\ell=1,\ldots,b$ and $\bx\in\bbR^d$,  
and $\bHhu$ and $\bH^\rU$ are positive definite matrices. 
Recall 
\begin{align*}
\bbeta^\ast_{}&=\argmin_{\bbeta\in\bbR^b}\; J_\PU(\bbeta) ,
\end{align*}
where
\begin{align*}
J_\PU (\bbeta)&=\frac{1}{2}\bbeta^\top\bH^\rU\bbeta^\top
	-\bbeta^\top\bh^\rP , \\
\bH^\rU&=\int\bphi(\bx)\bphi(\bx)^\top p(\bx)\dbx, \\
\bh^\rP&=\int\bphi(\bx)p(\bx\mid y=+1)\dbx .
\end{align*}
Similarly,
\begin{align*}	
\bbetah&=\argmin_{\bbeta\in\bbR^b}\; \Jh_\PU(\bbeta) ,
\end{align*} 
where
\begin{align*}
\Jh_\PU (\bbeta)&=\frac{1}{2}\bbeta^\top\bHhu\bbeta^\top
	-\bbeta^\top\bhhp ,	\\
\bHhu&=\frac{1}{\nun}\sum_{k=1}^\nun \bphi(\bxu_k)\bphi(\bxu_k)^\top , \\
\bhhp&=\frac{1}{\np}\sum_{i=1}^\np \bphi(\bxp_i) .
\end{align*}

Firstly, we have the following lemma:
\begin{lemma}
\label{lem:sec-growth-cond}
Let $\epsilon$ be the smallest eigenvalue of $\bH^\rU$. 
We have 
\begin{align*}
J_\PU(\bbeta)\geq 
	J_\PU(\bbeta^\ast) + \epsilon\|\bbeta-\bbeta^\ast\|_2^2 .	
\end{align*}
That is, $J_\PU$ satisfies the second order growth condition 
\citep{SIAM:Bonnans+Shapiro:1998}.
\end{lemma}
\begin{proof}
Since $\bH^\rU$ is positive definite, $J_\PU(\bbeta)$ is strongly convex with
parameter at least $\epsilon$.
Then, we have
\begin{align*}
J_\PU(\bbeta) &\geq J_\PU(\bbeta^\ast) 
	+ \nabla J_\PU(\bbeta^\ast)^\top
	(\bbeta - \bbeta^\ast) + \epsilon\|\bbeta - \bbeta^\ast\|_2^2 \\
&= J_\PU(\bbeta^\ast) + \epsilon\|\bbeta - \bbeta^\ast\|_2^2 ,
\end{align*}
where the optimality condition 
$\nabla J_\PU(\bbeta^\ast)=\bzero$ is used.
\end{proof}
Let us define a set of perturbation parameters as
\begin{align*}
\calU:=\{\bU^\rU,\; \bu^\rP  \mid 
	\bU^\rU\in\bbS^b,\; \bu^\rP\in\bbR^b\} ,
\end{align*}
where $\bbS^b$ is the set of symetric $b\times b$ matrices.   
With these perturbation parameters, we express $\bHhu$ and $\bhhp$
as $\bU^\rU=\bHhu-\bH^\rU$ and $\bu^\rP=\bhhp-\bh^\rP$, respectively.
Let $\bu\in\calU$. Our perturbed objective function and the solution are given by
\begin{align*}
J_\PU(\bbeta, \bu)&:=\frac{1}{2}\bbeta^\top(\bH^\rU + \bU^\rU)\bbeta
	-\bbeta^\top(\bh^\rP + \bu^\rP) , \\ 
\bbeta(\bu)&:=\argmin_{\bbeta\in\bbR^b}\; J_\PU(\bbeta,\bu) . 
\end{align*}
Apparently, $J_\PU(\bbeta)=J_\PU(\bbeta, \bzero)$.
Also, $\Jh_\PU(\bbeta)=J_\PU(\bbeta,\bu)$ and $\bbetah=\bbeta(\bu)$ for $\bu\neq\bzero$.
We then have the following Lemma:
\begin{lemma}
\label{lem:lip-mod}
$J_\PU(\cdot, \bu)-J_\PU(\cdot)$ is Lipschitz continuous modulus
$\omega(\bu)=\calO(\|\bU^\rU\|_\mathrm{Fro} + \|\bu^\rP\|_2)$,
where $\|\cdot\|_\mathrm{Fro}$ is the Frobenius norm. 
\end{lemma}
\begin{proof}
Firstly, we have
\begin{align*}
J_\PU(\bbeta,\bu)-J_\PU(\bbeta)
	=\frac{1}{2}\bbeta^\top\bU^\rU\bbeta
	- \bbeta^\top\bu^\rP .
\end{align*}
The partial gradient is given by
\begin{align*}
\frac{\partial}{\partial \bbeta}(J_\PU(\bbeta,\bu)-J_\PU(\bbeta))
	=\bU^\rU\bbeta - \bu^\rP .
\end{align*}
Let us define the $\delta$-ball of $\bbeta^\ast$ as
$\calB_\delta(\bbeta^\ast):=\{\bbeta\mid \|\bbeta-\bbeta^\ast\|_2\leq\delta\}$,
and $M=\|\bbeta^\ast\|_2$.
For any $\bbeta\in\calB_\delta(\bbeta^\ast)$,
we can easily show
\begin{align*}
\|\bbeta\|_2 \leq \|\bbeta-\bbeta^\ast\|_2 + \|\bbeta^\ast\|_2 \leq \delta + M ,
\end{align*}
where we first used the triangle inequality and 
then $\|\bbeta-\bbeta^\ast\|_2^{}\leq\delta$ and $M=\|\bbeta^\ast\|_2^{}$.
Thus, 
\begin{align*}
&\Big\|\frac{\partial}{\partial \bbeta}
	(J_\PU(\bbeta,\bu)-J_\PU(\bbeta))\Big\|_2 %\\
	\leq (\delta + M)\|\bU^\rU\|_\mathrm{Fro}
	+ \|\bu^\rP\|_2 .
\end{align*}
This means that $J_\PU(\cdot, \bu)-J_\PU(\cdot)$ 
is Lipschitz continuous
on $\calB_\delta(\bbeta^\ast)$ with a Lipschitz constant of order
$\calO(\|\bU^\rU\|_\mathrm{Fro}+\|\bu^\rP\|_2)$.
\end{proof}
Finally, we prove Theorem~\ref{thm:rate}.
\begin{proof}
According to the \emph{central limit theorem}, we have 
\begin{align*}
\|\bU^\rU\|_\mathrm{Fro}=\calO_p(1/\sqrt{\nun}), ~~~
\|\bu^\rP\|_2=\calO_p(1/\sqrt{\np})
\end{align*}
as $\np,\nun\to\infty$.
Since we proved that $J_\PU$ satisfies the second order growth condition 
(Lemma~\ref{lem:sec-growth-cond})
and $J_\PU(\cdot,\bu)-J_\PU(\cdot)$ is Lipschitz continuous modulus $\omega(\bu)$
(Lemma~\ref{lem:lip-mod}), we can use Proposition $6.1$ in \citet{SIAM:Bonnans+Shapiro:1998}
and have the first half of Theorem~\ref{thm:rate}:
\begin{align*}
\|\bbetah - \bbeta^\ast\|_2
&\leq \epsilon^{-1}\omega(\bu) \\
&=\calO(\|\bU^\rU\|_\mathrm{Fro}+\|\bu^\rP\|_2) \\
&=\calO_p(1/\sqrt{\np}+1/\sqrt{\nun}) .
\end{align*}

Next, we prove the latter half of Theorem~\ref{thm:rate}. 
For the squared errors, we have
\begin{align*}
\big|\Jh_\PU(\bbetah) - J_\PU(\bbeta^\ast)\big| 
&\leq \big|\Jh_\PU(\bbetah) - \Jh_\PU(\bbeta^\ast)\big| 
	+ \big|\Jh_\PU(\bbeta^\ast) - J_\PU(\bbeta^\ast)\big| .
\end{align*}
Here, we have
\begin{align*}
\Jh_\PU(\bbetah) - \Jh_\PU(\bbeta^\ast)
&=\frac{1}{2}(\bbetah + \bbeta^\ast)^\top\bHhu(\bbetah - \bbeta^\ast) 
	-(\bbetah - \bbeta^\ast)^\top\bhhp , \\
\Jh_\PU(\bbeta^\ast) - J_\PU(\bbeta^\ast)
&=\frac{1}{2}\bbeta^{\ast\top}\bU^\rU\bbeta^\ast - \bu^\rP\bbeta^\ast 
\end{align*}
Since $0\leq\phi_\ell(\bx)\leq 1$ and $M=\|\bbeta^\ast\|_2$,
it leads to 
\begin{align*}
|\Jh_\PU(\bbetah) - J_\PU(\bbeta^\ast)| 
&\leq |\Jh_\PU(\bbetah) - \Jh_\PU(\bbeta^\ast)|
	+ |\Jh_\PU(\bbeta^\ast) - J_\PU(\bbeta^\ast)| \\
&\leq \calO_p(\|\bbetah - \bbeta^\ast\|_2)
	+ \calO_p(\|\bU^\rU\|_\mathrm{Fro} + \|\bu^\rP\|_2) \\
&=\calO_p(1/\sqrt{\np} + 1/\sqrt{\nun}) .	
\end{align*}
Recall 
\begin{align*}
\PUSMI^\ast&=\frac{\thetap}{\thetan}
	\Big(-J_\PU(\bbeta^\ast)-\frac{1}{2}\Big) , \\
\widehat{\PUSMI}&=\frac{\thetap}{\thetan}
	\Big(-\Jh_\PU(\bbetah)-\frac{1}{2}\Big) .
\end{align*}
We thus have
\begin{align*}
|\PUSMI^\ast - \PUSMIh|
&=\frac{\thetap}{\thetan} |\Jh_\PU(\bbetah) - J_\PU(\bbeta^\ast)| \\ 
&=\calO_p(1/\sqrt{\np} + 1/\sqrt{\nun}) .
\end{align*}
This concludes the theorem.
\end{proof}

\section{Effect of Dimension Reduction}
\label{sec:effect-dim-red}
In this section, we illustrate the effect of dimension reduction
and how the number of samples affects class-prior estimation.

We use the artificial dataset used in Section~\ref{sec:rep-learn}
and vary both $n_\rP^{}$ and $n_\rU^{}$ from $500$ to $5{,}000$. 
We set the true class-prior $\theta_\rP^{}$ as $0.5$.
The class-prior is estimated by the method based on 
kernel mean embedding (KM) \citep{ICML:Ramaswamy+etal:2016}.
To evaluate the performance with and without dimension reduction,
we use the one-dimensional samples obtained by $\bb^\top\bx$,
where $\bb=(1, 0)^\top$, and the original two-dimensional samples.

\begin{figure}[t]
	\centering
	\subfigure[Mean absolute error (with its standard error) 
	between the true and estimated class-priors over $10$ trials.]{%
 		\includegraphics[clip, width=.47\columnwidth]{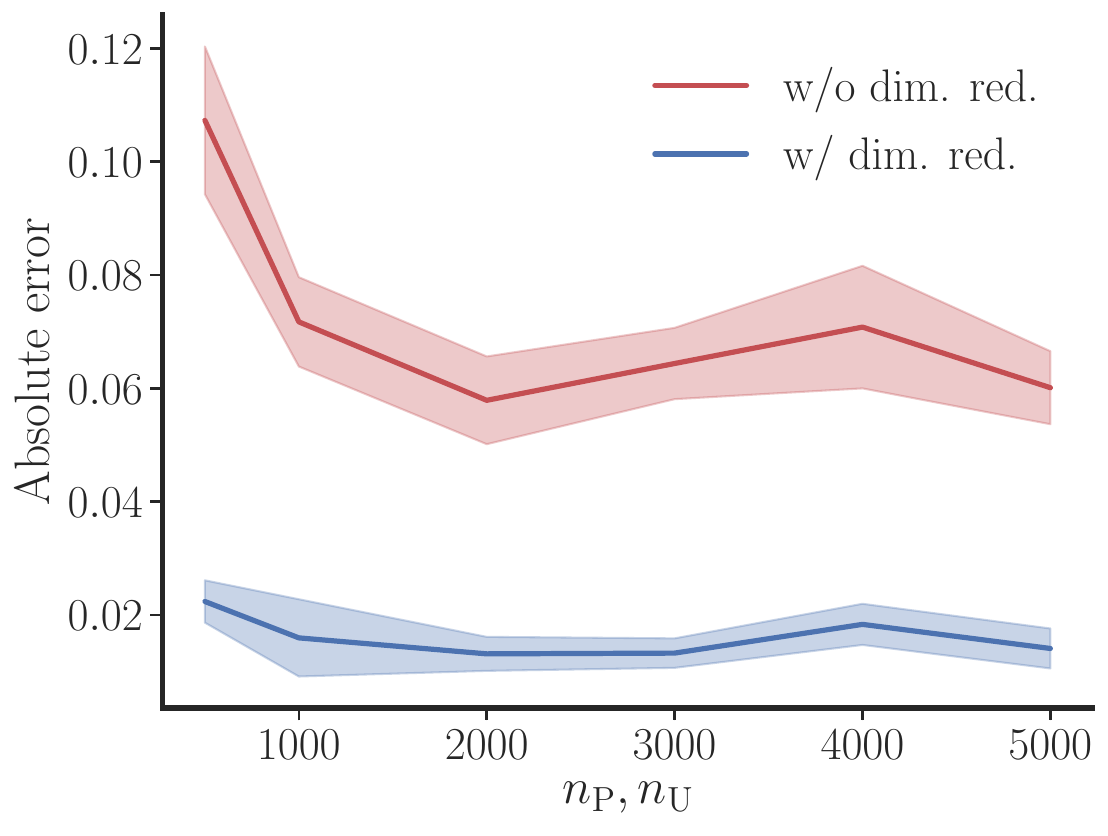}
		\label{fig:large-size-priorh}		
	}\hspace*{7mm}%
	\subfigure[Mean computation time \textrm{[}sec.\textrm{]} 
	(with its standard error) over $10$ trials.]{%
 		\includegraphics[clip, width=.47\columnwidth]{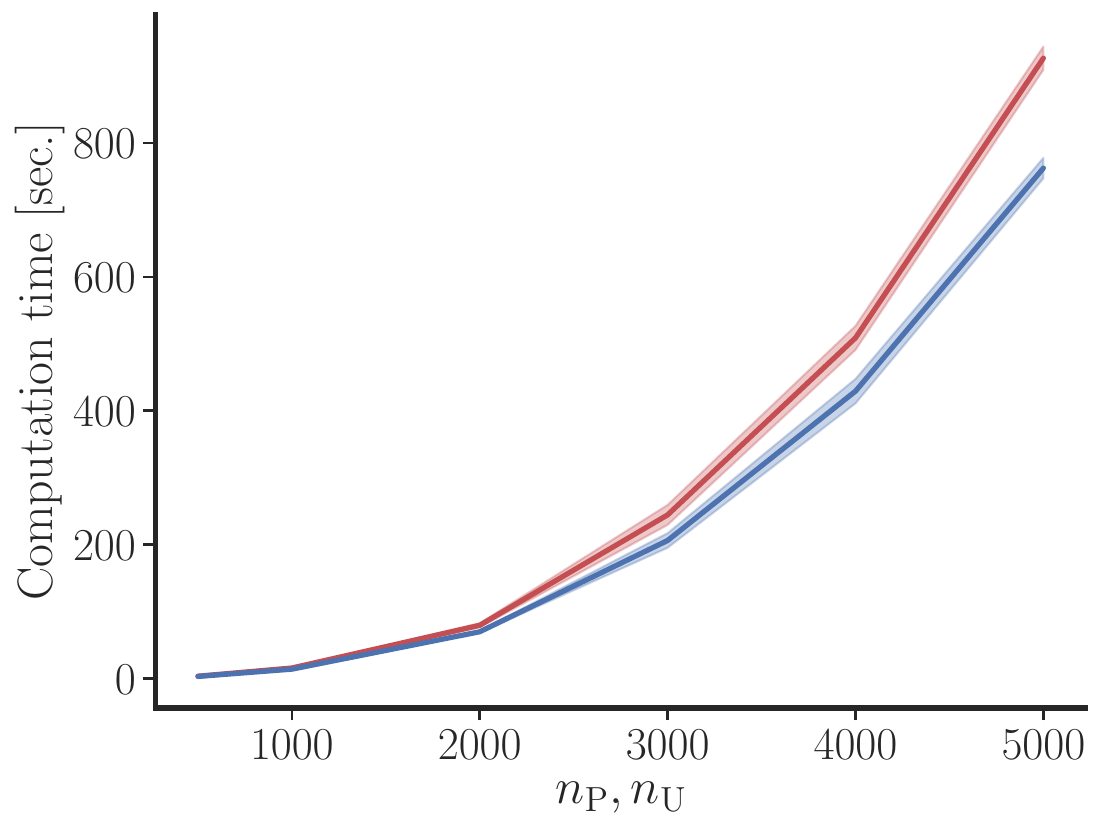} 		
		\label{fig:large-size-time}
	}%
	\caption{
	Although the error of KM without dimension reduction decreases
	the number of samples until around $n_\rP=n_\rU=2{,}000$,				
	the error of KM with dimension reduction 
	is smaller than that without dimension reduction.
	For the computation time, it grows with the number of samples.
	Since the short computation time and low absolute error
	are desirable, this result shows the effectiveness of 
	dimension reduction.  			
	}
	\label{fig:large-size}
\end{figure}

Figure~\ref{fig:large-size-priorh} shows 
the mean absolute error (with its standard error) between
the true and estimated class-priors over $10$ trials.
The error of KM without dimension reduction decreases
the number of samples until around $n_\rP=n_\rU=2{,}000$,
but the error is not reduced even if we increase $n_\rP=n_\rU=5{,}000$.  				
In contrast, at $n_\rP=n_\rU=500$, 
the error of KM with dimension reduction 
is already smaller than that without dimension reduction.
Figure~\ref{fig:large-size-time} shows
the mean computation time (with its standard error) over $10$ trials.
The computation time grows with the number of samples. 
Since the short computation time and low absolute error
are desirable, this result shows the effectiveness of dimension reduction.

\section{Supervised counterpart of the proposed method}
\label{sec:pn-smi}
In this section, we review the SMI estimation method 
\citep{BMCBio:Suzuki+etal:2009a,Entropy:Sugiyama:2013,IEICE:Sakai+Sugiyama:2014}.
 
According to \citet{NeCo:Suzuki:2013}, SMI can be exrepssed as 
\begin{align}
\SMI=\frac{1}{2}\sum_{ y\{\pm1\} }\int r^2(\bx,y)p(\bx)p(y)\dbx - \frac{1}{2} ,
\label{eq:another_form_of_smi}
\end{align}
where 
\begin{align*}
r(\bx,y):=\frac{p(\bx,y)}{p(\bx)p(y)} .
\end{align*}
Based on the \emph{Fnechel inequality} \citep{BOOK:Boyd+Vandenberghe:2004},
for any function $h\colon\bbR^d\times\{\pm1\}\to\bbR$, we have
\begin{align}
\frac{1}{2}r^2(\bx,y)\geq h(\bx,y)r(\bx,y)-\frac{1}{2}h^2(\bx,y) ,
\end{align}
where the equality condition is $h(\bx,y)=r(\bx,y)$.
We thus obtain the lower bound of SMI by
\begin{align}
\SMI\geq L(h) - \frac{1}{2}, 
\end{align}
where 
\begin{align}
L(h):=\sum_{ y\in\{\pm1\} }\int h(\bx,y)p(\bx,y)\dbx
	- \frac{1}{2}\sum_{ y\in\{\pm1\} }\int h^2(\bx,y)p(\bx)p(y)\dbx .
\end{align}
To obtain an SMI estimate, we first train $h$ with PN data by solving 
the following optimization problem: 
\begin{align}
\maximize_{h\in\mathcal{H}}\; \widehat{L}(h) ,
\end{align}
where $\mathcal{H}$ is a user-specified function class and 
$\widehat{L}$ is sample approximation of $L$, i.e.,
\begin{align}
\widehat{L}(h):=\frac{1}{n}\sum_{i=1}^n h(\bx_i,y_i)-
	\sum_{ y\in\{\pm1\} } \frac{\ph(y)}{2n}\sum_{i=1}^n h^2(\bx_i, y) - \frac{1}{2} .
\end{align}
A simple approach to approxiamtion of $\ph(y)$ is to use
the number of labeled samples, i.e., $\ph(y=+1)$ is approximated by
the number of positive samples divided by that of all labeled samples.
The SMI approximator from PN data is then given by
\begin{align*}
\PNSMIh:=\widehat{L}(\widehat{h})- \frac{1}{2} ,	
\end{align*}    
where $\widehat{h}:=\argmax_{h\in\mathcal{H}}\;\widehat{L}(h)$.
 
Similarly to the proposed PURL method, we maximize $\PNSMIh$ 
to learn low-dimensional representation.

\end{document}